%% file: UAI.tex
\documentclass[letterpaper]{article}
\usepackage{uai2020}
\usepackage[margin=1in]{geometry}

\usepackage{times}
\usepackage{amsmath}
\usepackage{natbib}
\usepackage{listings}
\usepackage{booktabs}
\usepackage{amssymb}
\usepackage{amsthm}

\usepackage{bbm}
\usepackage{graphicx}
\usepackage{booktabs} % for professional tables
\usepackage{subcaption}
\usepackage{graphics}
\usepackage{mathrsfs} 
\input{math_commands.tex}

\newtheorem{definition}{Definition}[section]
\newtheorem{theorem}{Theorem}
\newtheorem{lemma}{Lemma}

\theoremstyle{remark}
\newtheorem*{remark}{Remark}

{
	\theoremstyle{plain}
	\newtheorem{assumption}{Assumption}
}

\usepackage{hyperref}% add hypertext capabilities
\usepackage[skip = 2pt, font = scriptsize]{caption}
\usepackage{algorithm}
\usepackage{algorithmic}
\usepackage{footnote}
\usepackage{booktabs}
\theoremstyle{definition}

\makeatletter
\newcommand{\printfnsymbol}[1]{%
  \textsuperscript{\@fnsymbol{#1}}%
}
\makeatother
\title{Deep Curvature Suite}

\author{ {\bf Diego Granziol\thanks{equal contribution}} \\
Oxford University \\
\And
{\bf Xingchen Wan\printfnsymbol{1}}  \\
Oxford University          \\
\And
{\bf Timur Garipov\printfnsymbol{1}}   \\
Samsung AI Moscow \\
}

\begin{document}

\maketitle

\begin{abstract}

Despite providing rich information into neural networks geometry and applications in Bayesian neural networks and second order optimisation, accessing curvature information is still a daunting engineering challenge and hence inaccessible to most practitioners. In some cases, proxy diagonal approximations, which we show on both real and synthetic examples can be arbitrarily bad. 
% Furthermore, iterative methods and their relationships to the problem of moments and the rich mathematical theorems underlying them are often so poorly understood, that they are often in-correctly implemented or their results misapplied. 
We hence provide an open-source software package to the community, to enable easy access to curvature information for real networks and datasets, not just toy examples and is typically an order of magnitude faster than competing packages. We address and disprove many common misconceptions in the literature, namely that the Lanczos algorithm learns eigenvalues from the top down. We also prove using high dimensional concentration inequalities that for specific classes of matrices a single random vector is sufficient for accurate spectral estimation. 
%In this work we both shed light on some of the misconceptions in the literature on iterative methods, which limit them from being used as state of the art spectral approximation tools and present the Deep Curvature suite, a PyTorch-based, open-source package for Hessian analysis and visualisation of neural network curvature. 
We showcase our package practical utility on a series of examples based on realistic modern neural networks tested on CIFAR-$10$/$100$ datasets.

% %We develop and implement a method to analyze the Hessian and Gauss-Newton spectra for Deep Neural Networks with millions of parameters. 
% Whilst neural networks classification results are outstanding, the study of their loss Hessians remains extremely challenging due to the size of the parameter space, typically in the order of millions. In this work we go beyond toy networks and propose an empirical study of more realistic models with millions of parameters
% Our contributions are three-fold: (i) We show that the number of outliers of the eigenspectrum of a neural network Hessian can be directly derived from the number of layers. (ii) Based on the stability of the observed spectra and the finite size corrections due to batch stochasticity, for which we provide a model, we propose a spectral method to learn the optimal learning rate and momentum at each epoch. iii) We create two algorithms, Spectral Stochastic Gradient Descent (SSGD) and Spectral Stochastic Gradient Descent with Momentum (SSGDM) and show that our finite size corrections are essential to effective optimization. We demonstrate our proposed method on logistic regression for notMNIST and for various %Basic and ResNet 
% convolutional neural network architectures on CIFAR-10.
\end{abstract}

\section{Introduction}

The success of deep learning models trained with gradient based optimizers on a range of tasks, from image classification/segmentation, natural language processing to reinforcement learning, often beating human level performance, has led to an explosion in the availability and ease of use of high performance software for their implementation. Automatic differentiation packages such as TensorFlow \citep{abadi2016tensorflow} and PyTorch \citep{paszke2017automatic} have become widely adopted, with higher level packages allowing users to state their model, dataset and optimiser in a few lines of code \citep{chollet2015}, effortlessly achieving state of the art performance.
\newline
\newline
However, the pace of development of software packages extracting second order information, representing the curvature at a point in weight space, has not kept abreast. Researchers aspiring to evaluate or use curvature information need to implement their own libraries, which are rarely shared or kept up to date. Naive implementations are computationally intractable for all but the smallest of models. Hence, researchers typically either completely ignore curvature information or use highly optimistic approximations, such as the diagonal elements of the matrix or of a surrogate matrix, with limited justification or empirical analysis of the harshness of the aforementioned approximations. 
\newline
\newline
Although the combination of fast Hessian vector products \citep{pearlmutter1994fast}, advanced linear algebraic techniques \citep{golub1994matrices} and high dimensional geometry \citep{hutchinson1990stochastic} holds the key to solving to making significant process in this space, the lack of focus given to these areas, means that implementations of these methods (to the extent that they exist at all), are either inefficient or incorrect. In this paper we
\begin{itemize}
    \item Motivate the use of curvature information, for both optimization, generalization, understanding the quality of optima, the effect of normalization techniques and evaluating theoretical assumptions
    \item Provide a primer on iterative methods, particularly the combination of the Lanczos algorithm with random seed vectors, typical pitfalls and errors in its implementation and interpretation. We highlight its link to orthogonal polynomials, the problem of moments, high dimensional concentration theorems and how it can be used to generate a highly representative spectral approximation with minimal computational or memory overhead
    \item Provide a primer on typical matrix assumptions used for evaluating spectra in deep learning, such as the diagonal or diagonal generalised Gauss-Newton approximation and evaluate their efficacy on both small random matrices corresponding to known eigenvalue distributions and deep neural networks
    \item We provide an open-source $2$nd-order PyTorch based software package, the  \textbf{Deep Curvature} suite\footnote{Available at \url{https://github.com/xingchenwan/MLRG_DeepCurvature}}, which allows for spectral visualisation of the Hessian and Generalised Gauss Newton, loss surface eigen-traversal, gradient, hessian and loss variance calculation on large expressive modern networks. We present examples of its usage on VGG networks \citep{simonyan2014very} and ResNets \citep{he2016deep} in a matter minutes on a single GPU. We further provide an implementation of iterative stochastic Newton methods for deep learning algorithms.
\end{itemize}

We use the GPytorch implementation \citep{gardner2018gpytorch} of the \textit{Lanczos algorithm} \citep{meurant2006lanczos}, which we introduce in Section \ref{sec:lanczos} and discuss the most common misconceptions in the literature in Section \ref{sec:misconceptionslanczos}. We also note that the GPyTorch implementation \citep{gardner2018gpytorch} has been used in a similar way to efficiently compute Hessian eigenspectra in \citet{izmailov2019subspace} and \citet{maddox2019simple}.

% Xingchen comment: is this title too cheeky?
%\section{Why should we care about Curvature?}
\section{The Importance of Curvature in Deep Learning}
The curvature at a point in weight-space informs us about the local conditioning of the problem, which determines the rate of convergence for first order methods and informs us about the optimal learning and momentum rates \citep{nesterov2013introductory}. The most common areas where curvature information is employed are analyses of the \textbf{Loss Surface} and \textbf{Newton} type methods in optimization. 

\subsection{Loss Surfaces}
Loss surface visualization of deep neural networks, have often focused on two dimensional slices of random vectors \citep{li2017visualizing} or the changes in the loss traversing a set of random vectors drawn from the $d$-dimensional Gaussian distribution \citep{izmailov2018averaging}. 
%CAN THIS BE CLEARLY SHOWN TO BE MISLEADING?
Recent empirical analyses of the neural network loss surfaces invoking full eigen-decomposition \citep{sagun2016eigenvalues,sagun2017empirical} have been limited to toy examples with $<5000$ parameters. Other works have used the diagonal of the Fisher information matrix \citep{chaudhari2016entropy}, an assumption we will challenge in this paper. Theoretical analysis relating the loss surface to spin-glass models from condensed matter physics and random matrix theory \citep{choromanska2015open,choromanska2015loss} rely on a number of unrealistic assumptions\footnote{Such as input independence.}.
Given that the spectra of many classes of random matrices are known \citep{tao2012topics,akemann2011oxford}, it may be helpful to visualise the spectra of large real networks and commonly used datasets to evaluate whether they match theoretical predictions. Other important areas of loss surface investigation include understanding the effectiveness of batch normalization \citep{ioffe2015batch}. Recent convergence proofs \citep{santurkar2018does} bound the maximal eigenvalue of the Hessian with respect to the activations and bounds with respect to the weights on a per layer basis. Bounds on a per layer basis do not imply anything about the bounds of the entire Hessian and furthermore it has been argued that the full spectrum must be calculated to give insights on the alteration of the landscape \citep{kohler2018exponential}. Recent work making curvature information more available, again through diagonal approximations, explicitly disallows the use of batch normalization \citep{dangel2019backpack}. Our software package extends seamlessly to batch normalization, allowing for evaluation in both train and eval mode.

% For large neural networks, with millions or billions of parameters, the $\mathcal{O}(N^{2})$ storage requirement and $\mathcal{O}(N^{3})$ eigendecomposition computational cost are completely infeasible. In order to analyse the spectra of large neural network Hessians, we use the Lanczos-Stieltjes \citep{bai1996some} method, and its relation to the method of moments, to approximate the Hessian empirical spectral density. 
% In order to avoid storing the matrix, we compute the Hessian vector multiplication using the Pearlmutter trick \citep{pearlmutter1994fast} $m\ll n$ times, where each Hessian vector product costs approximately $2$ gradient evaluations. 
% We analyse the Hessian and Gauss-Newton spectral evolution during training with stochastic gradient descent (SGD) for different deep network architectures, the effect of batch-norm on the spectrum. We show empirically that problems considered difficult for SGD, but not for Hessian-free methods, such as autoencoders, have spectral widths orders of magnitude larger than problems for which SGD (with momentum) is considered sufficient. We also show that the use of batch-norm decreases the spectra width by approximately $2$ orders of magnitude and provide an argument as to how a reduced spectral width allows for higher learning rate and increased convergence %{\color{red}{WHAT ARE WE SHOWING?? WHAT ARE THE MAIN RESULTS??}}

\subsection{Newton Methods in Deep Learning}
%Second order optimisation for deep neural networks offers several advantages over stochastic gradient descent (SGD). %Firstly, it mitigates the vanishing/exploding gradients problem by leveraging curvature information of the loss function. Secondly, 
% In cases where the optimisation is plagued by locally dependent, tightly coupled parameters, or large variations in scale along different vectors in the parameter space, second-order methods can be expected to significantly outperform first order methods. Practical implementations, such as the Hessian free method \citep{martens2010deep,dauphin2014identifying, vinyals2012krylov}, continue to be useful in a plethora of applications \citep{graff2014skynet,chung2017parallel}, requiring orders of magnitude fewer iterations than SGD. 
All second order methods solve the minimisation problem for the loss, $L$ associated with parameters $\vp$ and perturbation $\vd$ to the second order in Taylor expansion,
\begin{eqnarray}
&\vd^{*} = \text{argmin}_{\vd} L(\vp+\vd) \nonumber\\
&L(\vp+\vd) =  L(\vp) + \nabla L^{T}\vd + \frac{1}{2}\vd^{T}\bar{\mH}\vd
\end{eqnarray}
Where instead of the true Hessian $\mH  = \nabla \nabla L(\vp) \in \mathbb{R}^{n\times n}$, a surrogate positive definite approximation to the Hessian $\bar{\mH}$, such as the Gauss-Newton, is employed so to make sure the minimum is lower bounded;  its solution is
\begin{equation}
\label{eq:update}
\vd = -\bar{\mH}^{-1}\nabla L(\vp) = -\sum_{i}^{N}\frac{1}{\lambda_{i}}\vu_{i}\vu_{i}^{T}\nabla L(\vp)
\end{equation}
where $\vu_{i}$ correspond to the generalised Hessian eigenvectors. The parameters are updated with $\vp = \vp - \alpha \vd$, in which $\alpha$ is the global learning rate. 

Despite the success of second order optimisation for difficult problems on which SGD is known to stall, such as recurrent neural networks \citep{martens2012training}, or auto-encoders \citep{Martens2016}. Researchers wanting to implement second order methods such as \citep{vinyals2012krylov,martens2012training,dauphin2014identifying} face the aforementioned problems of difficult implementation. As a minor contribution, we also include two stochastic Lanczos based optimisers in our code. We plot the training error of the VGG-$16$ network on CIFAR-$100$ dataset against epoch in Figure \ref{fig:secondorder}. We keep the ratio of damping constant to learning rate constant, where $\delta = 10 \alpha$, for a variety of learning rates in $\{1,0.1,0.01,0.001,0.0001\}$ with a batch size of $128$ for both the gradient and the curvature, all of which post almost identical performance. We also compare against different learning rates of Adam, both of which converge significantly slower per iteration compared to our stochastic Newton methods, and we in black plot SGD, with a typical learning rate of $0.05$ and momentum $0.9$, which has an unstable optimisation trajectory.
\begin{figure}
    \centering
    \includegraphics[trim=0cm 0cm 0cm 0cm, clip, width=1\linewidth]{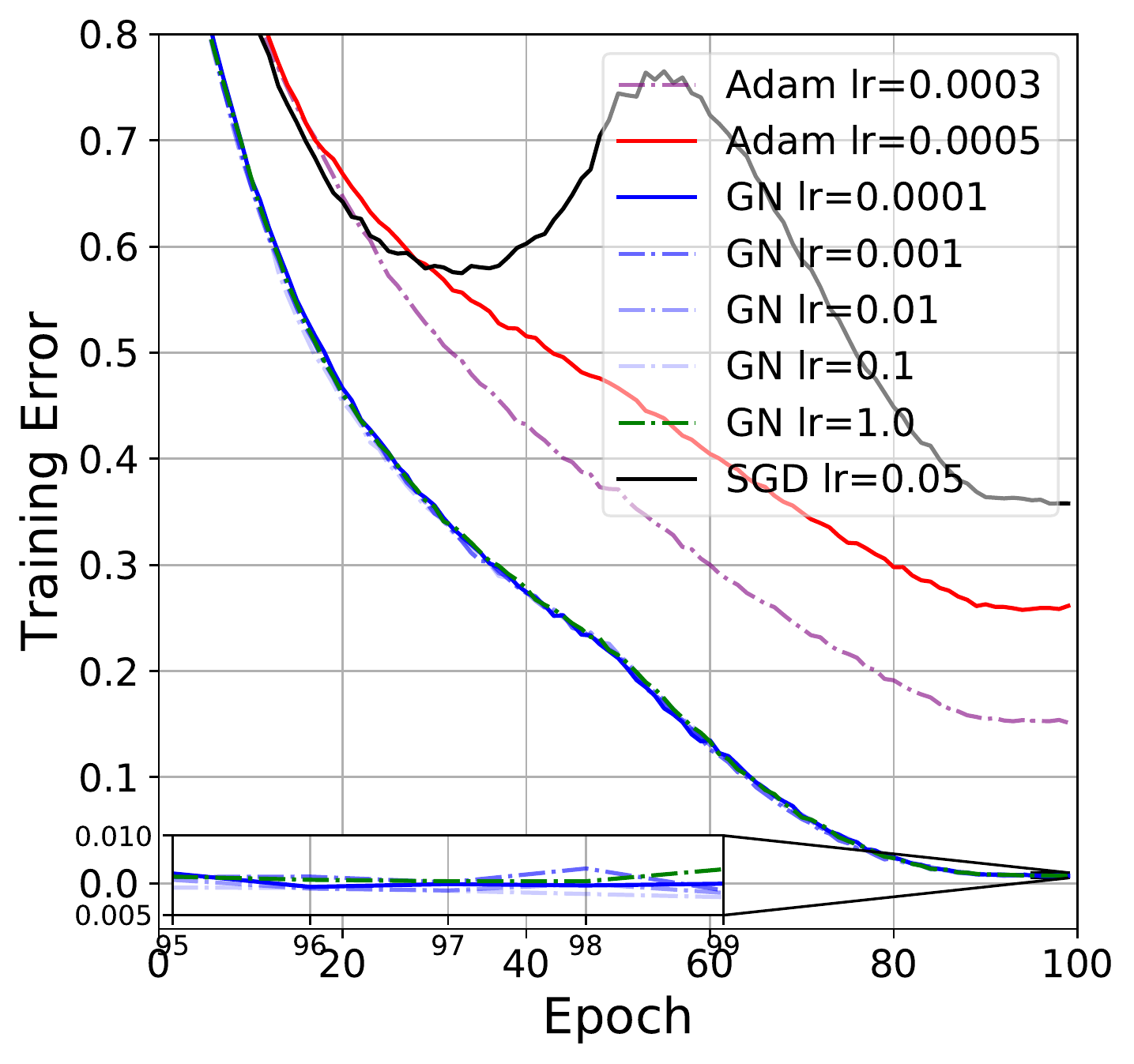}
    \caption{Training Error of stochastic Lanczos Newton methods on CIFAR-$100$ VGG-$16$ against baselines.}
    \label{fig:secondorder}
\end{figure}

% , one which uses the absolute eigenvalues of the Hessian \citep{dauphin2014identifying} and another which uses the Generalised-Gauss-Newton approximation, which is similar to the conjugate gradient method of \citep{martens2012training}. 

\paragraph{Bayesian Neural Networks} As a minor remark, we note that Bayesian neural networks use the Laplace approximation, featuring the inverse of the Hessian multiplied by a vector \citep{bishop2006pattern}. Our code allows for an estimation of this quantity, which may also be of use to the community and serve as an alternative for KFAC-Laplace \citep{ritter2018a}.

% With the note-able exception of the KFAC algorithm \citep{martens2015optimizing}, which computes a block diagonal approximation to the inverse Fisher information matrix updating its approximation every few steps to keep computation tractable

\section{Learning to love Lanczos}
\label{sec:lanczos}
% IS the original title too cheeky?

%Should we just call this section "Laczos Algorithm"?
% \section{Learning to Love Lanczos}
% The Lanczos algorithm, has a rich history \citep{meurant2006lanczos,golub1994matrices,van1983matrix}, with elegant proofs of convergence, error analysis for finite precision, relationships to Gauss-quadrature rules, orthogonal polynomials and the problem of moments. Given that Lanczos algorithm, for which we use the GPU implementation \citep{gardner2018gpytorch}, is one of the key theoretical foundations of our proposed tool and is often subjected to various misconceptions commonly seen in literature, we summarise some of its main properties here. 

%\subsection{Introduction}

The Lanczos Algorithm (Algorithm \ref{alg:lanczos}) is an iterative algorithm for learning a subset of the eigenvalues/eigenvectors of any Hermitian matrix, requiring only matrix vector products. It can be regarded as a far superior adaptation of the power iteration method, where the Krylov subspace $\mathscr{K}(\mH,\vv) = \mathrm{span}\{\vv, \mH^\vv,\mH^{2}\vv...\}$ is orthogonalised using Gram-Schmidt. 
% Consider a symmetric matrix $\mH \in \R^{P\times P}$ 
% and a random vector $\vv \in \R^{1\times P}$, where we can write 
% $\vv = \sum_{i=1}^{P}\alpha_{i}\vu$, where $\sum_{i=1}^{N}\alpha_{i}^{2}=1$ and $\vu$ are the eigenvectors of $\mH$. As $N \rightarrow \infty$, $\forall i$, $ \thinspace P(\alpha_{i}=0)  \rightarrow 0$. Labelling the eigenvalues in decreasing order $\lambda_{1}\geq\lambda_{2}...$,
% \begin{eqnarray}
% \mH^{m}\vv &=& \lambda_{1}^{m}\biggl(\alpha_{1}\vw_1+\sum_{i\neq  1}^{N}\alpha_{i}\biggl(\frac{\lambda_{i}}{\lambda_{1}}\biggr)^{m}\vw_i\biggr)\nonumber\\
% && \xrightarrow[m\rightarrow \infty]{|\lambda_{2}|<|\lambda_{1}|} \lambda_{1}\alpha_{1}\vw_1,\label{poweriteration}
% \end{eqnarray}
% which, when normalised, gives the eigenvector/eigenvalue pair corresponding to $\lambda_{1}$. For a sufficiently large spectral gap $|\lambda_{1}/\lambda_{2}|$, Equation \eqref{poweriteration} faithfully approximates the largest eigenvalue and eigenvector with very few iterations $m$. 
Beyond having improved convergence to the power iteration method \citep{bai1996some} by storing the intermediate orthogonal vectors in the corresponding Krylov subspace, %makes use of the repeated matrix vector multiplications by additionally
Lanczos produces estimates of the eigenvectors and eigenvalues of smaller absolute magnitude, known as Ritz vectors/values. Despite its known superiority to the power iteration method and relationship to orthogonal polynomials and hence when combined with random vectors the ability to estimate the entire spectrum of a matrix, these properties are often ignored or forgotten by practitioners, we hence include a full primer in Appendix \ref{sec:whateveryoneshouldknowlanczos}. We also explicitly debunk some key persistent myths in the next section.

%We hence expect the largest eigenvalues and their corresponding eigenvectors to be most faithfully captured within a small number $m\ll N$ of iterations. It has been shown \citep{newman1970vibration} that we should use $3/2$ as many steps as we wish accurate eigenvalue/eigenvector estimates. 

\section{Common Misconceptions}
\label{sec:misconceptionslanczos}
% Due to its involved mathematical nature relative to its simpler power iteration counterpart, the Lanczos algorithm, as applied use in machine learning, is unfortunately subjected a range of misconceptions and misuse. 
% %Because the our software package relies strongly on the Lanczos algorithm, as we advocate for broader usage of Lanczos, we feel it is imperative to state some of the common misconceptions to avoid potential misinterpretations of our results, and results from Lanczos algorithm in general. 
% Here, we highlight two most common  misconceptions:

\begin{itemize}
    \item We can learn the negative and interior eigenvalues by shifting and inverting the matrix sign $\mH \rightarrow -\mH +\mu\mI$
        \item Lanczos learns the largest $m$ $[\lambda_{i},\vu_{i}]$ pairs of $\mH \in \mathbb{R}^{P\times P}$ with high probability \citep{dauphin2014identifying}
\end{itemize}
Since these two related beliefs are prevalent, we disprove them explicitly in this section, with Theorems \ref{theorem:idiotsoflanczosmyth1} and \ref{theorem:idiotsoflanczosmyth2}.
\begin{theorem}
\label{theorem:idiotsoflanczosmyth1}
The shift and invert procedure $\mH \rightarrow -\mH + \mu \mI$, changes the Eigenvalues of the Tri-diagonal matrix $\mT$ (and hence the Ritz values) to $\lambda_{i} = -\lambda_{i}+\mu$
\end{theorem}
\begin{proof}
Following the equations from Algorithm \ref{alg:lanczos}
\begin{equation}
\begin{aligned}
    & \vw_{1}^{T} = (-\mH+\mu\mI)\vv_{1} \thinspace \thinspace \&~\alpha_{1} = \vv_{1}^{T}\mH\vv_{1} + \mu\mI \\
    & \vw_{2} = \vw_{1}-\alpha_{1}\vv_{1} 
    = (\mH + \mu\mI)\vv_{1}-(\vv_{1} ^{T}\mH\vv_{1}  + \mu\mI)\vv_{1} \\
    & \vw_{2} = (\mH - \vv_{1}^{T}\mH\vv_{1})\vv_{1} \thinspace \thinspace \&~\vv_{2} = \vw_{2}/||\vw_{2}|| \\
    & \alpha_{2} = \vv_{2}^{T}(-\mH+\mu \mI)\vv_{2} = -\vv_{2}^{T}\mH\vv_{2} + \mu\\
    & \beta_{2} = ||\vw_{2}||
\end{aligned}
\end{equation}
Assuming this for $m-1$, and repeating the above steps for $m$ we prove by induction and finally arrive at the modified tridiagonal Lanczos matrix $\Tilde{\mT}$
\begin{equation}
\begin{aligned}
    & \Tilde{\mT} = -\mT + \mu \mI\\
    & \tilde{\lambda_{i}} = -\lambda_{i}+\mu \thinspace \thinspace \forall 1 \leq i \leq m
\end{aligned}
\end{equation}
\end{proof}
\begin{remark}
 No new Eigenvalues of the matrix $\mH$ are learned. Although it is clear that the addition of the identity does not change the Krylov subspace, such procedures are commonplace in code pertaining to papers attempting to find the \textit{smallest eigenvalue}. This disproves the first misconception.
\end{remark}

\begin{theorem}
\label{theorem:idiotsoflanczosmyth2}
For any matrix $\mH \in \mathbb{R}^{P \times P}$ such that $\lambda_{1} > \lambda_{2} > ..... >\lambda_{P}$ and $\sum_{i=1}^{m}\lambda_{i} < \sum_{i=m+1}^{P}\lambda_{i}$ in expectation over the set of random vectors $\vv$ the $m$ eigenvalues of the Lanczos Tridiagonal matrix $\mT$ do not correspond to the top $m$ eigenvalues of $\mH$
\end{theorem}
\begin{proof}
Let us consider the matrix $\tilde{\mH} = \mH - \frac{\lambda_{m+1}+\lambda_{m}}{2}\mI$, 
\begin{equation}
     \begin{cases}
    \lambda_{i} > 0 ,& \forall i \leq m \\
    \lambda_{i} < 0 ,& \forall i > m \\
    \end{cases}
\end{equation}
Under the assumptions of the theorem, $\text{Tr}(\tilde{\mH})<0$ and hence by Theorem \ref{theorem:lanczosspectrum} and Equation \ref{eq:stochtrace} there exist no $w_{i}>0$ such that
\begin{equation}
\label{eq:idiotsoflanczosp1}
\sum_{i=1}^{m}w_{i}\lambda_{i}^{k} = \frac{1}{P} \sum_{i=1}^{P}\lambda_{i}^{k} \forall \thinspace \thinspace \thinspace \thinspace \thinspace 1 \leq k \leq m
\end{equation}
is satisfied for $k=1$, as the LHS is manifestly positive and the RHS is negative. By Theorem \ref{theorem:idiotsoflanczosmyth1} this holds for the original matrix $\mH$.
\end{proof}
\begin{remark}
Given that Theorem \ref{theorem:idiotsoflanczosmyth2} is satisfied over the expectation of the set of random vectors, which by the CLT is realised by Monte Carlo draws of random vectors as $d \rightarrow \infty$ the only way to really span the top $m$ eigenvectors is to have selected a vector which lies in the $m$ dimensional subspace of the $P$ dimensional problem corresponding to those vectors, which would correspond to knowing those vectors \emph{a priori}, defeating the point of using Lanczos at all.
\end{remark}
Another way to see this is Theorem \ref{theorem:lanczoseigenvalues}, which gives a bound on the distance between the smallest Lanczos-Ritz value and the minimal eigenvalue. 
Intuitively, as the Ritz values and weights form a discrete $m$-moment spectral approximation to the Hessian spectrum, hence the support of the discrete density, cannot approximately match the largest $m$ eigenvalues. 
This can be seen in Figure \ref{subfig:wignerstem10000} where we run Lanczos with $m=30$ steps and capture the shape of the spectral density of a $\mH \in \mathbb{R}^{10000\times 10000}$ matrix including the negative eigenvalue of largest magnitude. 

\section{Deep Curvature}

Based on the Lanczos algorithm, we are in a position to introduce to our package, the \textbf{Deep Curvature suite}, a software package that allows analysis and visualisation of deep neural network curvature. The main features and functionalities of our package are:

\begin{itemize}
    \item \textbf{Network training and evaluation} we provide a range of pre-built modern popular neural network structures, such as VGG and variants of ResNets, and various optimisation schemes in addition to the ones already present in the PyTorch frameworks, such as K-FAC and SWATS. These facilitates faster training and evaluation of the networks (although it is worth noting that any PyTorch-compatible optimisers or architectures can be easily integrated into our analysis framework).
    
    \item \textbf{Eigenspectrum analysis of the curvature matrices} Powered by the Lanczos techniques implemented in
GPyTorch \citep{gardner2018gpytorch} and outlined in Section 3, \textit{with a single random vector} we use the Pearlmutter matrix-vector product trick for fast inference of the eigenvalues and eigenvectors of the common curvature matrices of the deep neural networks. In addition to the standard Hessian matrix, we also include the feature for inference of the eigen-information of the Generalised Gauss-Newton matrix, a commonly used positive-definite surrogate to Hessian\footnote{The computation of the GGN-vector product is similar with the computational cost of two backward passes in the network. Also, GGN uses \textit{forward-mode automatic differentiation} (FMAD) in addition to the commonly employed \textit{backward-mode automatic differentiation} (RMAD). In the current PyTorch framework, the FMAD operation can be achieved using two equivalent RMAD operations.}. 
    
    \item \textbf{Advanced Statistics of Networks} In addition to the commonly used statistics to evaluate network training and performance such as the training and testing losses and accuracy, we support computations of more advanced statistics: For example, we support squared mean and variance of gradients and Hessians (and GGN), squared norms of Hessian and GGN, L2 and L-inf norms of the network weights and etc. These statistics are useful and relevant for a wide range of purposes such as the designs of second-order optimisers and network architecture.
    
    \item \textbf{Visualisations} For all main features above, we include accompanying visualisation tools. In addition, with the eigen-information obtained, we also feature visualisations of the \textbf{loss landscape} by studying the sensitivity of the neural network to perturbations of weights. While similar tools have been available, we would like to emphasise that one key difference is that, instead of the \textit{random} directions as featured in some other packages, we explicitly perturb the weights in the \textit{eigenvector} directions, which should yield more informative results. 

\end{itemize}
\paragraph{Package Structure} The main interface functions are organsed as followed:
\begin{itemize}
    \item \textbf{./core} The functions under \texttt{core} directories are the main analysis tools of the package. \texttt{train\_network} allows network training and saving of the required statistics for subsequent spectrum learning. Based on the output of it, we additionally include tools for spectrum analysis (\texttt{compute\_eigenspectrum}) and advanced loss statistics (such as covariance of gradients and second order information like Hessian variance) in \texttt{compute\_loss\_stats} and \texttt{build\_loss\_landscape}. 
    
    We provide some pre-built network architectures (such as VGG and ResNet architectures) and optimizers apart from PyTorch natives (such as K-FAC, SWATS optimizers). We additionally support Stochastic Weight Averaging proposed in \citep{izmailov2018averaging}. However, it is worth noting that any PyTorch compatible networks and optimizers can be easily integrated in our framework. 
    
    \item \textbf{./visualise} This directory defines the various pre-defined visualisation functions for different purposes, including the visualisation of training, spectrum and the loss landscape.
\end{itemize} 

To facilitate a quick start of our package, we have included an illustrated example of analysis on the VGG-16 network on CIFAR-100 dataset in Appendix \ref{appendix:example}.

\section{Examples on Small Random Matrices}
In this section, we use some examples on small random matrices to showcase the power of our package that uses the Lanczos algorithm with random vectors to learn the spectral density. Here, we look at known random matrices with elements drawn from specific distributions which converge to known spectral densities in the asymptotic limit. Here we consider \textbf{Wigner Ensemble} \citep{wigner1993characteristic} and the \textbf{Marcenko Pastur} \citep{marchenko1967distribution}, both of which are extensively used in simulations or theoretical analyses of deep neural network spectra \citep{pennington2017geometry,choromanska2015loss,anonymous2020towards}.

\subsection{Wigner Matrices}
Wigner matrices can be defined in Definition \ref{def:wigner}, and their distributions of eigenvalues are governed by the semi-circle distribution law (Theorem \ref{def:wignersemicircle}).

\begin{theorem}
\label{def:wignersemicircle}
Let $\{M_{P}\}^{\infty}_{P=1}$ be a sequence of Wigner matrices, and for each $P$ denote $X_{P} = M_{P}/\sqrt{P}$. Then $\mu_{X_{P}}$, converges weakly, almost surely to the semi circle  distribution,
\begin{equation}
    \sigma(x)dx = \frac{1}{2\pi}\sqrt{4-x^{2}}\mathbf{1}_{|x|\leq 2}
\end{equation}
\end{theorem}

For our experiments, we generate random matrices $\mH \in \mathbb{R}^{P \times P}$ with elements drawn from the distribution $\mathcal{N}(0,1)$ for $P = \{225,10000\}$ and plot histogram of the spectra found by eigendecomposition, along with the predicted Wigner density (scaled by a factor of $\sqrt{P}$) in Figures \ref{subfig:wignerhist225} \& \ref{subfig:wignerhist10000} and compare them along with the discrete spectral density approximation learned by lanczos in $m=30$ steps using a single random vector $d=1$ in Figures \ref{subfig:wignerstem225} \& \ref{subfig:wignerstem10000}. It can be seen that even for a small number of steps $m \ll P$ and a single random vector, Lanczos impressively captures not only the support of the eigenvalue spectral density but also its shape. We note as discussed in section \ref{sec:misconceptionslanczos} that the $30$ Ritz values here do not span the top $30$ eigenvalues even approximately. 
\begin{figure}
	\centering
	\begin{subfigure}{0.45\linewidth}
	\centering
    \includegraphics[trim=2.5cm 8.1cm 2.5cm 8.4cm, clip, width=1\linewidth]{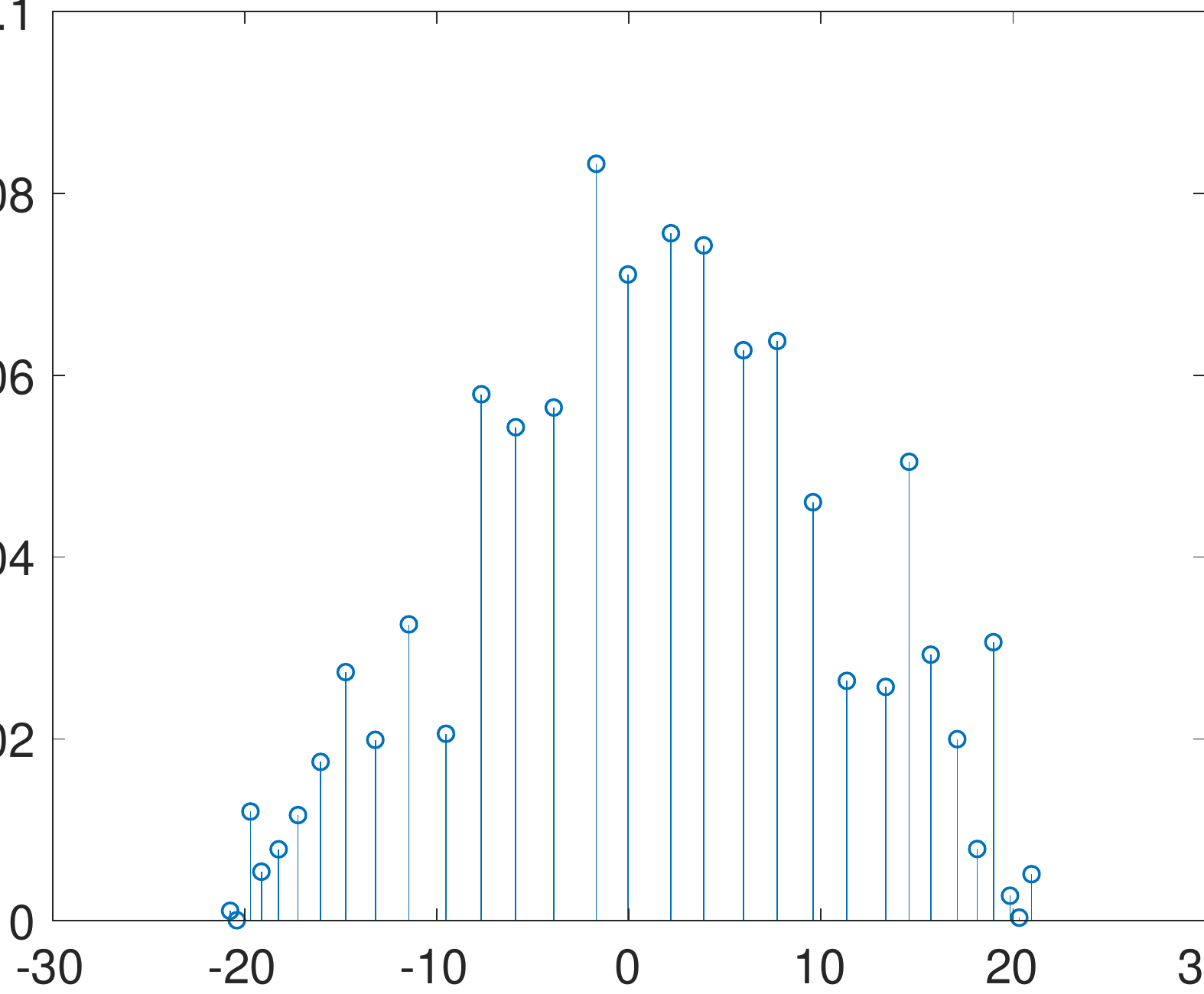}
	\caption{Stem $m=30, P=225$}
	\label{subfig:wignerstem225}	
	\end{subfigure}
	\hspace{5pt}
	\begin{subfigure}{0.45\linewidth}
	\centering
	\includegraphics[trim=2.5cm 8.1cm 2.5cm 8.4cm, clip, width=1\linewidth]{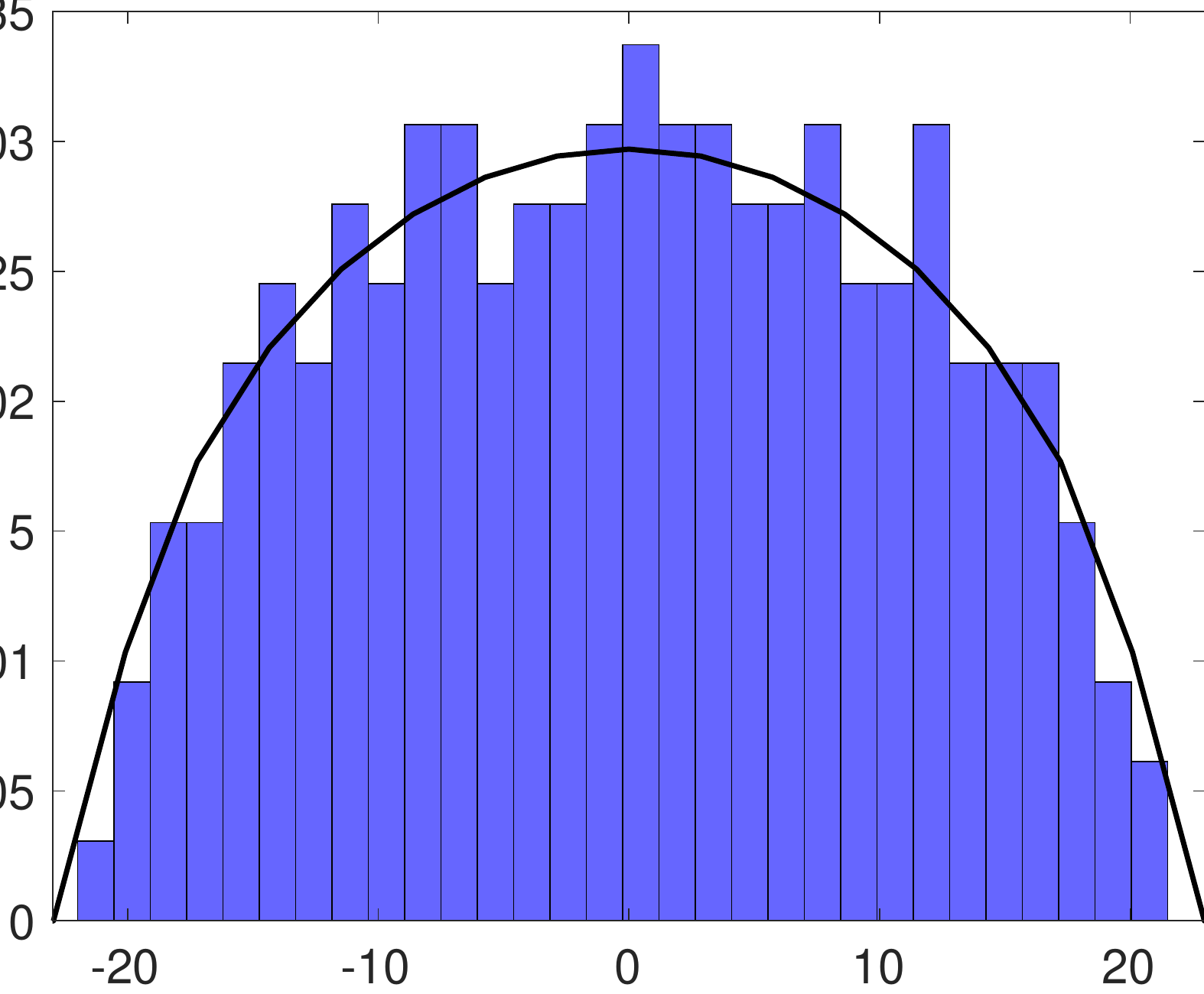}
	\caption{Histogram,$P=225$}
	\label{subfig:wignerhist225}		
	\end{subfigure}
	\begin{subfigure}{0.45\linewidth}
	\centering
	\includegraphics[trim=2.5cm 8.1cm 2.5cm 8.4cm, clip, width=1\linewidth]{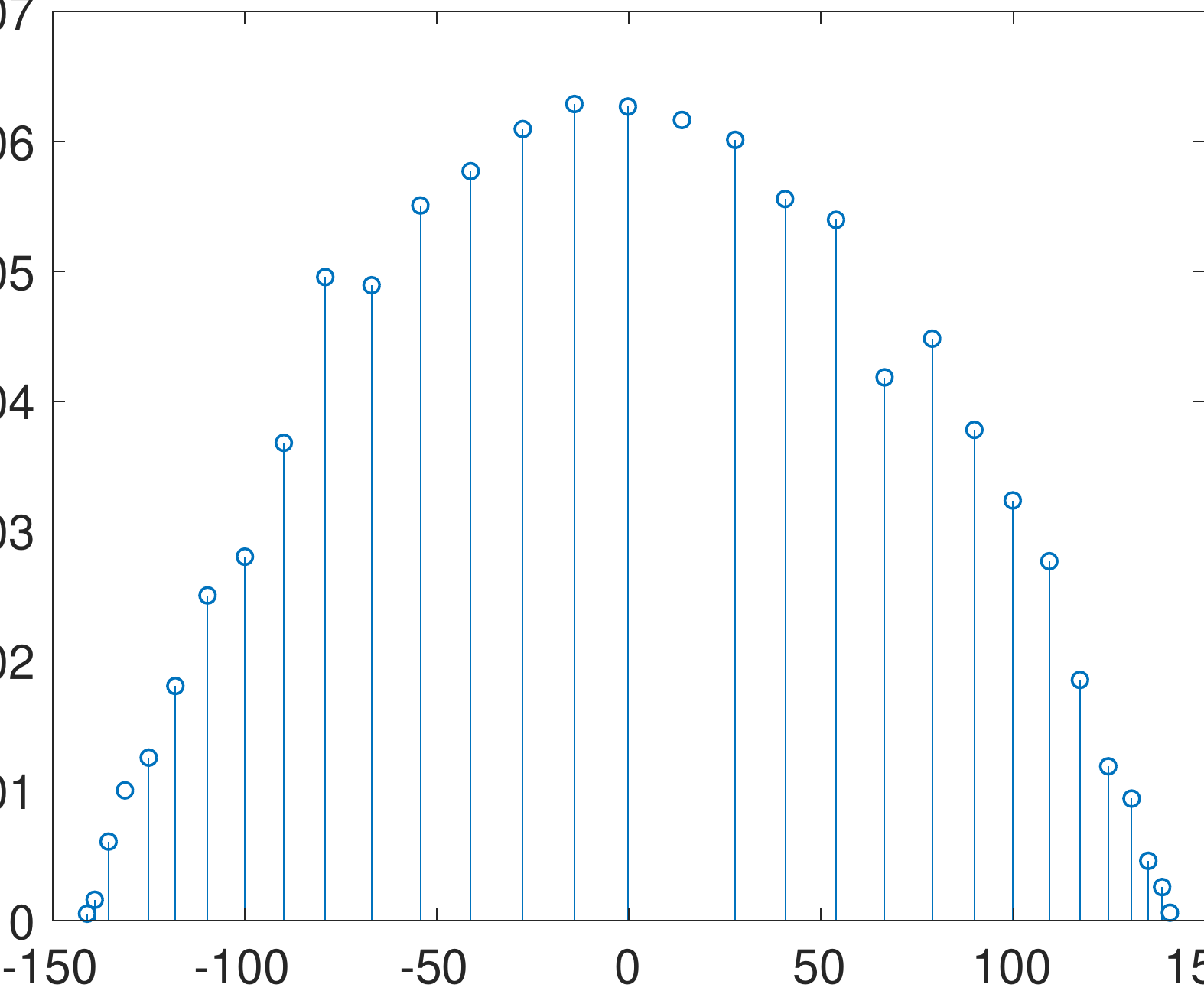}
	\caption{Stem $m=30, P=10^{4}$}
	\label{subfig:wignerstem10000}		
	\end{subfigure}
	\begin{subfigure}{0.45\linewidth}
	\centering
	\includegraphics[trim=2.5cm 8.1cm 2.5cm 8.4cm, clip, width=1\linewidth]{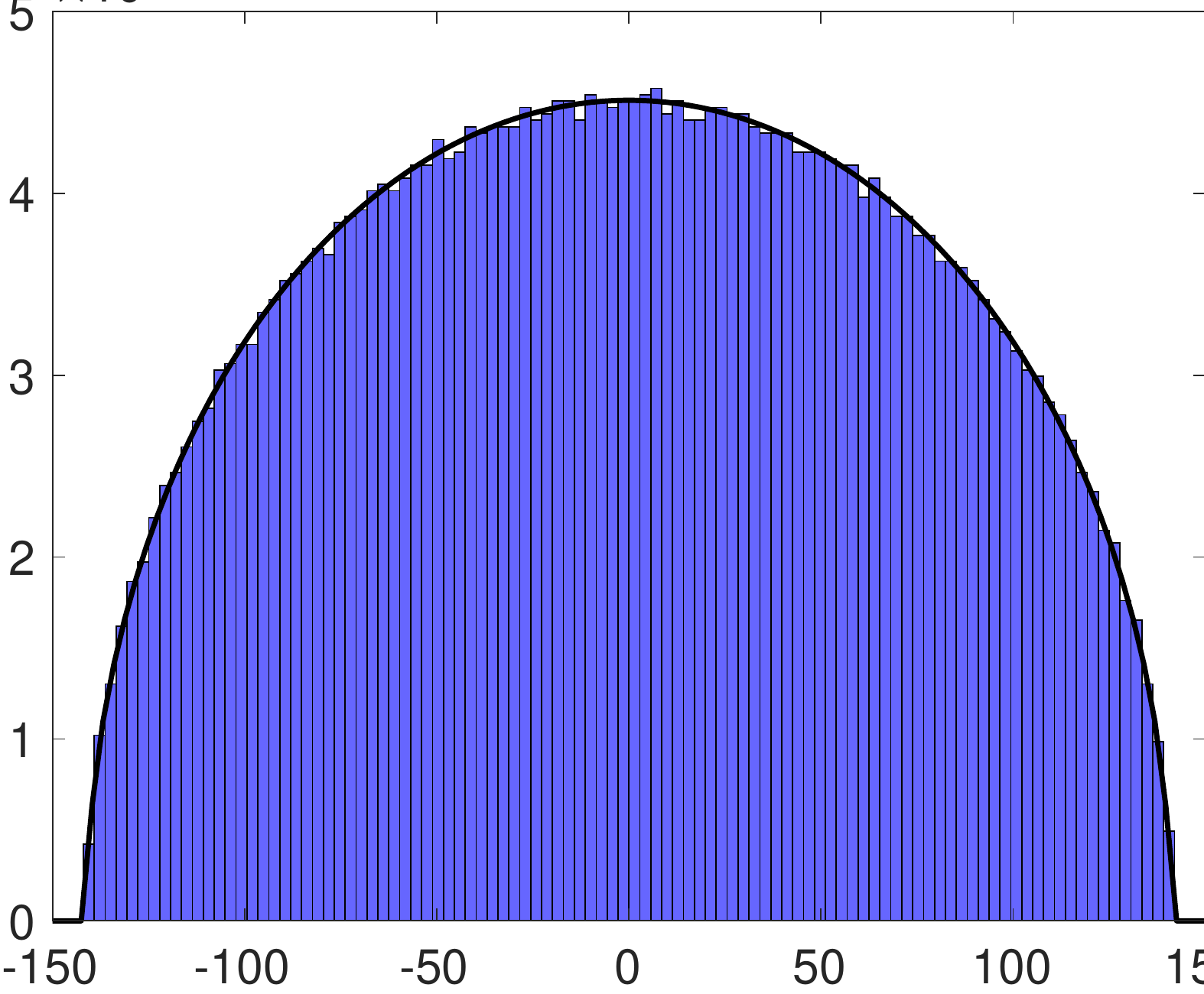}
	\caption{Histogram, $P=10^{4}$}
	\label{subfig:wignerhist10000}		
	\end{subfigure}
	\caption{Lanczos stem plot for a single random vector with $m=30$ steps compared to actual eigenvalue histogram for matrices of the form $\mH \in \mathbb{R}^{P \times P}$, where each element is a drawn from a normal distribution with unit variance, converging to the Wigner semi circle.}
	\label{fig:wignerlanchist}
\end{figure}
\label{sec:lanczosstemhist}

\subsection{Marcenko-Pastur}
An equally important limiting law for the limiting spectral density of many classes of matrices constrained to be positive definite, such as covariance matrices, is the Marcenko-Pastur law \citep{marchenko1967distribution}. Formally, given a matrix $\mX \in \mathbb{R}^{P\times T}$ with i.i.d zero mean entires with variance $\sigma^{2}<\infty$. Let $\lambda_{1}\geq \lambda_{2},...\geq\lambda_{P}$ be eigenvalues of $\mY_{n} = \frac{1}{T}\mX\mX^{T}$. The random measure $\mu_{P}(A) = \frac{1}{P}\#\{\lambda_{j} \in A\}$, $A \in \mathbb{R}$
\begin{theorem}
\label{theorem:mplaw}
Assume that $P,N \rightarrow \infty$ and the ratio $P/N \rightarrow q \in (0,\infty)$ (this is known as the Kolmogorov limit) then $\mu_{P}\rightarrow \mu$ in distribution where
\begin{equation}
  \begin{cases}
    (1-\frac{1}{q})\mathbbm{1}_{0\in A}+\nu_{1/q}(A) ,& \text{if } q> 1 \\
    \nu_{q}(A) ,& \text{if } 0\leq q \leq 1 
\end{cases} 
\end{equation}
\begin{equation}
\begin{aligned}
    & d\nu_{q} = \frac{\sqrt{(\lambda_{+}-x)(x-\lambda_{-})}}{\lambda x 2\pi\sigma^{2}}, \lambda_{\pm} = \sigma^{2}(1\pm \sqrt{q})^{2}
    \end{aligned}
\end{equation}
\end{theorem}

Here, we construct a random matrix $\mX \in \mathbb{P\times T}$ with independently drawn elements from the distribution $\mathcal{N}(0,1)$ and then form the matrix $\frac{1}{T}\mX\mX^{T}$, which is known to converge to the Marcenko-Pastur distribution. We use $P = \{225,10000\}$ and $T = 2P$ and plot the associated histograms from full eigendecomposition in Figures \ref{subfig:MPhist225} \& \ref{subfig:MPhist10000} along with their $m=30, d=1$ Lanczos stem counterparts in Figures \ref{subfig:MPstem225} \& \ref{subfig:MPstem10000}. Similarly we see a faithful capturing not just of the support, but also of the general shape. 
We note that both for Figure \ref{fig:wignerlanchist} and Figure \ref{fig:mplanchist}, the smoothness of the discrete spectral density for a single random vector increases significantly, even relative to the histogram. 
\begin{figure}
	\centering
	\begin{subfigure}{0.45\linewidth}
	\centering
    \includegraphics[trim=2.5cm 8.1cm 2.5cm 8.4cm, clip, width=1\linewidth]{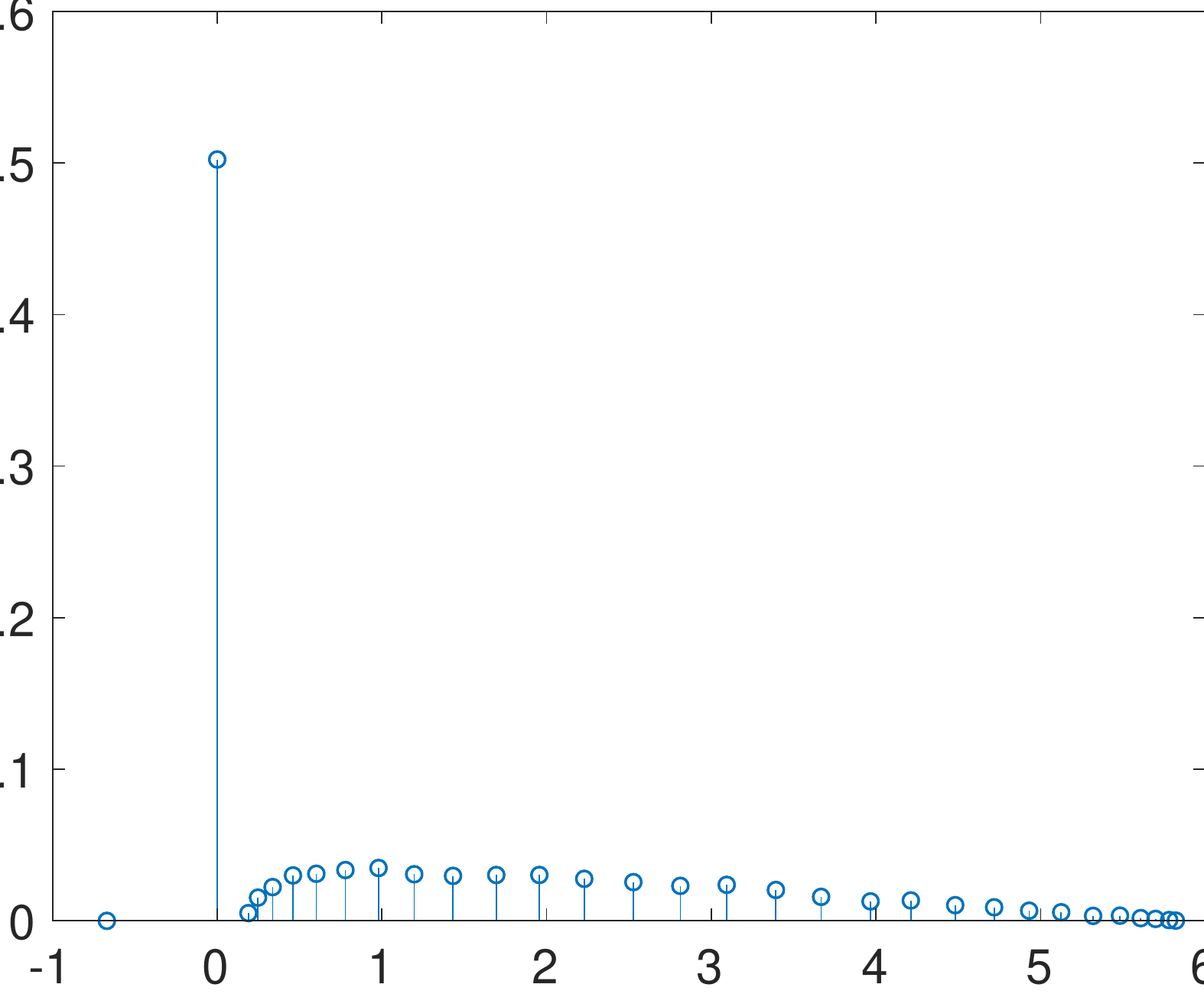}
	\caption{Stem $m=30, P=10^{4}$}
	\label{subfig:MPstem_q0p5}	
	\end{subfigure}
	\hspace{5pt}
	\begin{subfigure}{0.45\linewidth}
	\centering
	\includegraphics[trim=2.5cm 8.1cm 2.5cm 8.4cm, clip, width=1\linewidth]{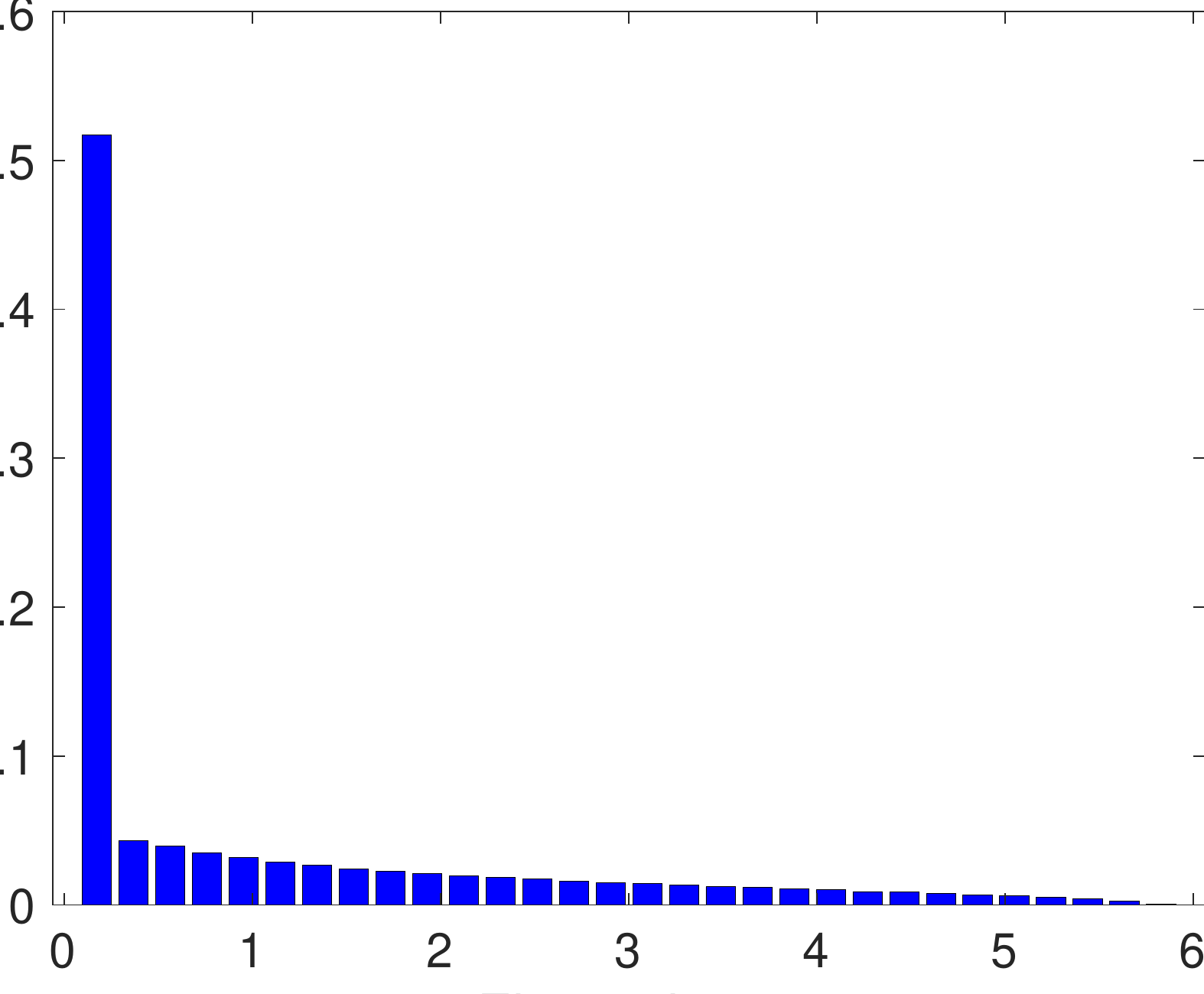}
	\caption{Histogram,$P=10^{4}$}
	\label{subfig:MPhist_q0p5}		
	\end{subfigure}
	\caption{Lanczos stem plot for a single random vector with $m=30$ steps compared to actual eigenvalue histogram for matrices of the form $\mH \in \mathbb{R}^{P \times P}$, where $\mH = \mX \mX^{T}/k$, where each element of $\mX^{P \times k}$, $k = 0.5P$ is a drawn from a normal distribution with unit variance, converging to the Marcenko-Pastur distribution with $q=0.5$.}
	\label{fig:mpandwignerlanchist2}
\end{figure}
\label{sec:mpstemq=0}

We also run the same experiment for $P=10000$ but this time with $T=0.5P$ so that exactly half of the eigenvalues will be $0$. We compare the Histogram of the eigenvalues in Figure \ref{subfig:MPhist_q0p5} against its $m=30, d=1$ Lanczos stem plot in Figure \ref{subfig:MPstem_q0p5} and find both the density at the origin, along with the bulk and support to be faithfully captured.

% \label{lanczos}

% \begin{figure}
% 	\centering
% 	\begin{subfigure}{0.45\linewidth}
% 	\centering
%     \includegraphics[trim=2.5cm 8.1cm 2.5cm 8.4cm, clip, width=1\linewidth]{toy_matrices/lanczos_stem_30_m=225.pdf}
% 	\caption{$\mH \in \mathbb{R}^{225\times225}$}
% 	\label{subfig:MPhist}	
% 	\end{subfigure}
% 	\hspace{5pt}
% 	\begin{subfigure}{0.45\linewidth}
% 	\centering
% 	\includegraphics[trim=2.5cm 8.1cm 2.5cm 8.4cm, clip, width=1\linewidth]{toy_matrices/lanczos_stem_30_m=1000.pdf}
% 		\caption{$\mH \in \mathbb{R}^{1000\times1000}$}
% 	\label{subfig:wignerhist}		
% 	\end{subfigure}
% 	\begin{subfigure}{0.45\linewidth}
% 	\centering
% 	\includegraphics[trim=2.5cm 8.1cm 2.5cm 8.4cm, clip, width=1\linewidth]{toy_matrices/lanczos_stem_30_m=5000.pdf}
% 		\caption{$\mH \in \mathbb{R}^{5000 \times 5000}$}
% 	\label{subfig:wignerhist}		
% 	\end{subfigure}
% 	\begin{subfigure}{0.45\linewidth}
% 	\centering
% 	\includegraphics[trim=2.5cm 8.1cm 2.5cm 8.4cm, clip, width=1\linewidth]{toy_matrices/lanczos_stem_30_m=10000.pdf}
% 		\caption{$\mH \in \mathbb{R}^{10000\times10000}$}
% 	\label{subfig:wignerhist}		
% 	\end{subfigure}
% 	\caption{Lanczos stem plot for a single random vector with $m=30$ steps for different dimension of matrices converging to the wigner semi circle.}
% 	\label{fig:mpandwignerlanchist}
% \end{figure}
% \label{sec:lanczoshist}
\subsection{Comparison to Diagonal Approximations}
\begin{figure}
	\centering
	\begin{subfigure}{0.45\linewidth}
	\centering
    \includegraphics[trim=2.5cm 8.1cm 2.5cm 8.4cm, clip, width=1\linewidth]{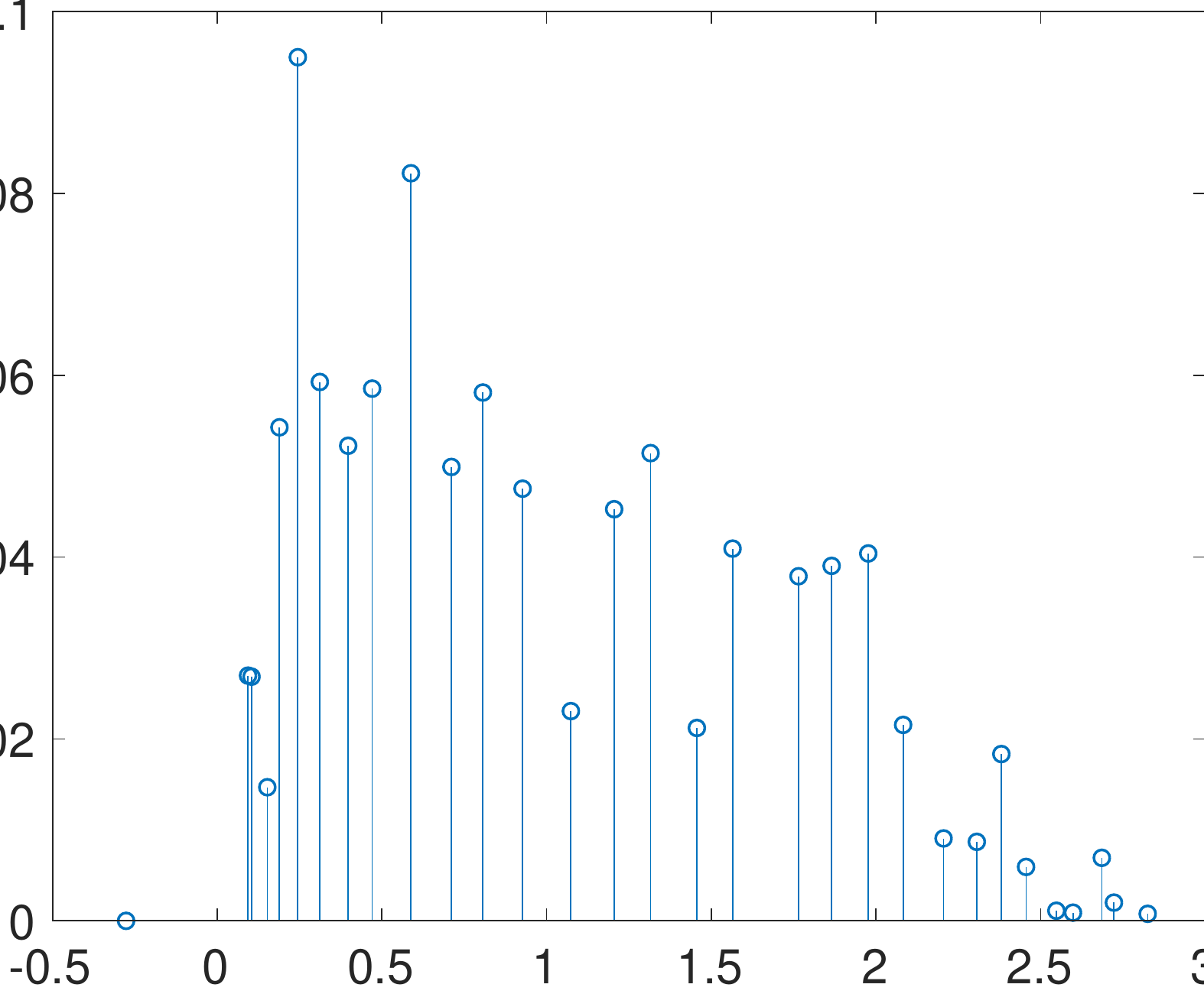}
	\caption{Stem $m=30, P=225$}
	\label{subfig:MPstem225}	
	\end{subfigure}
	\hspace{5pt}
	\begin{subfigure}{0.45\linewidth}
	\centering
	\includegraphics[trim=2.5cm 8.1cm 2.5cm 8.4cm, clip, width=1\linewidth]{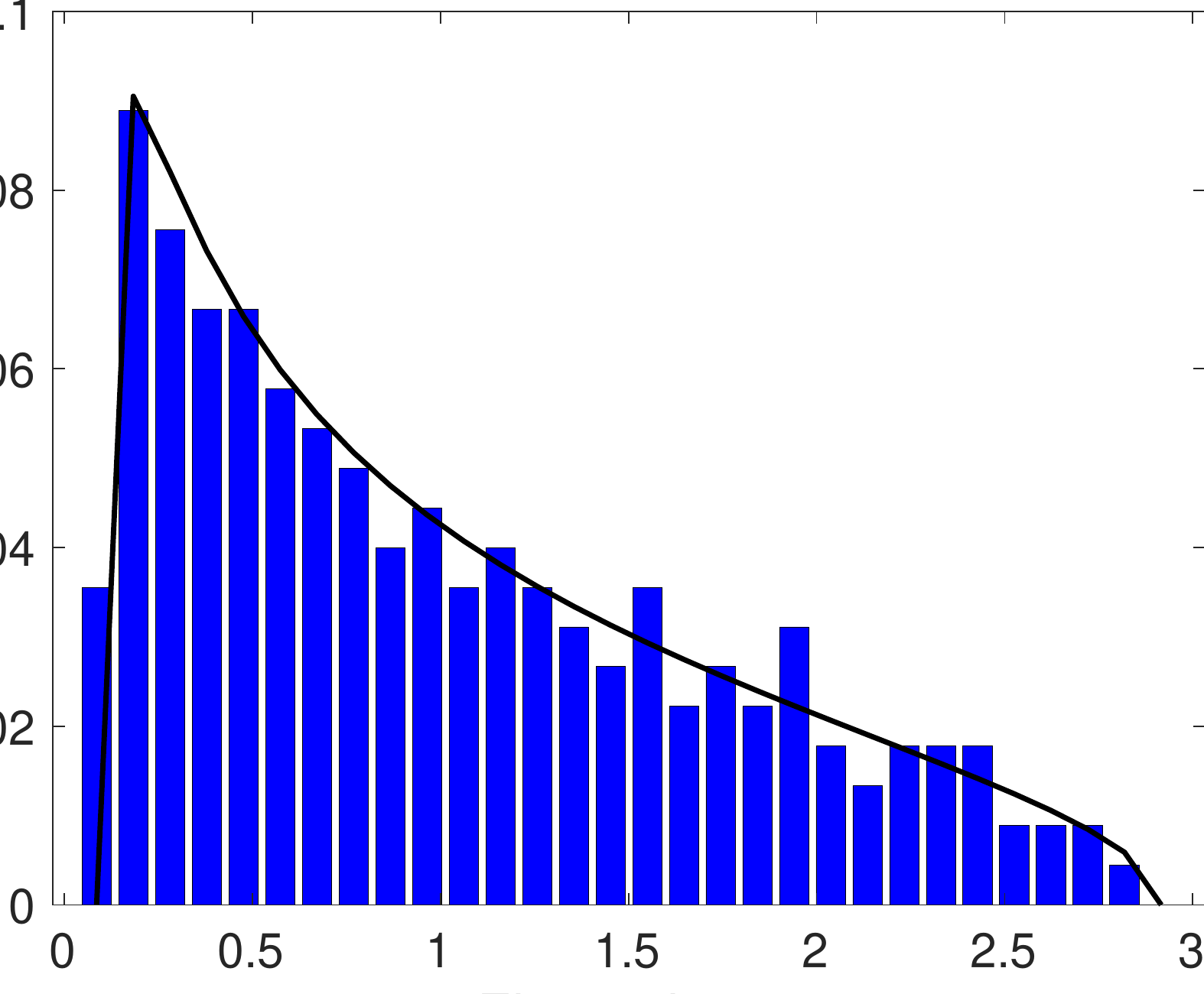}
	\caption{Histogram,$P=225$}
	\label{subfig:MPhist225}		
	\end{subfigure}
	\begin{subfigure}{0.45\linewidth}
	\centering
	\includegraphics[trim=2.5cm 8.1cm 2.5cm 8.4cm, clip, width=1\linewidth]{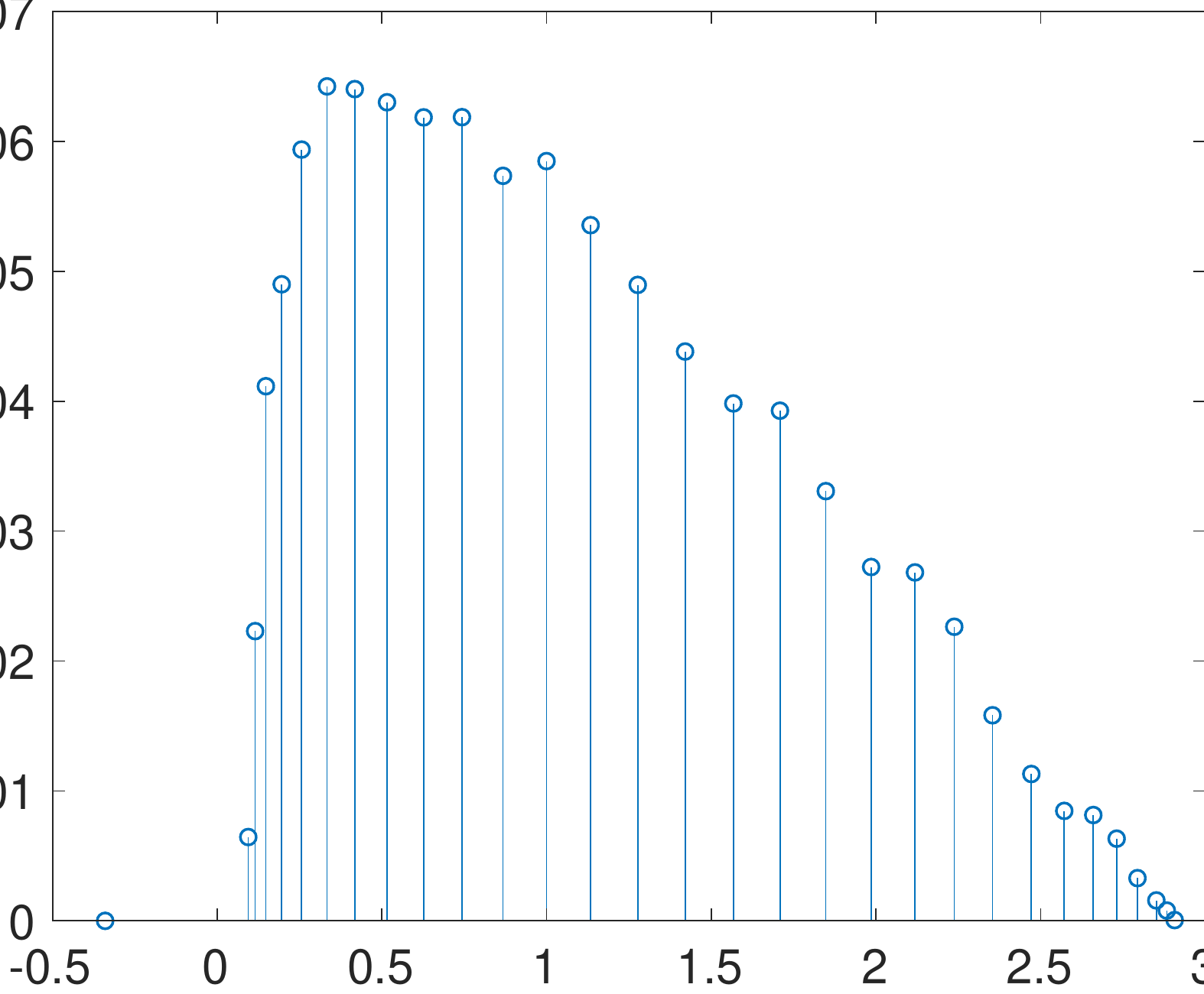}
	\caption{Stem $m=30, P=10^{4}$}
	\label{subfig:MPstem10000}		
	\end{subfigure}
	\begin{subfigure}{0.45\linewidth}
	\centering
	\includegraphics[trim=2.5cm 8.1cm 2.5cm 8.4cm, clip, width=1\linewidth]{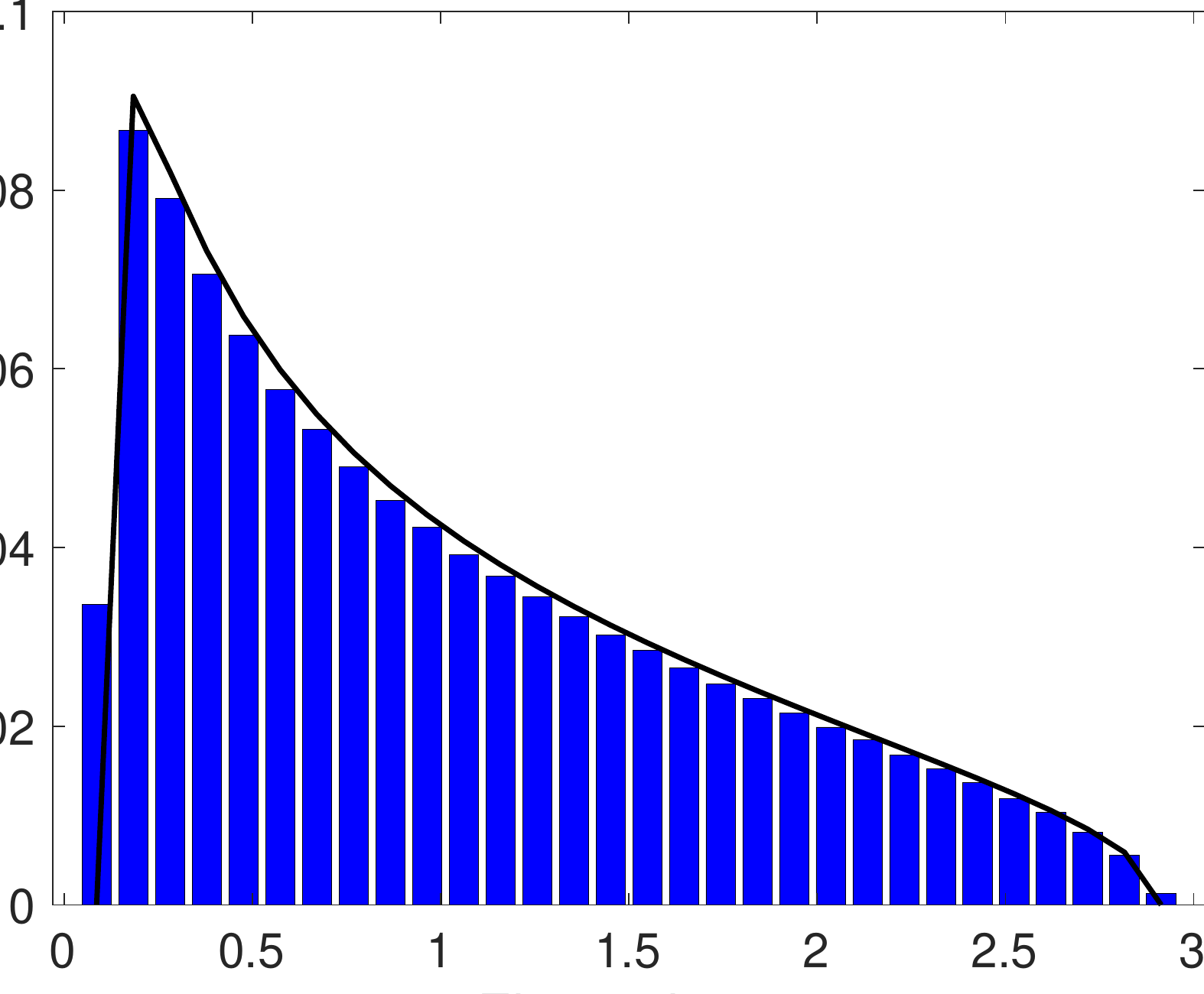}
	\caption{Histogram, $P=10^{4}$}
	\label{subfig:MPhist10000}		
	\end{subfigure}
	\caption{Lanczos stem plot for a single random vector with $m=30$ steps compared to actual eigenvalue histogram for matrices of the form $\mH \in \mathbb{R}^{P \times P}$, where $\mH = \mX \mX^{T}/k$, where each element of $\mX^{P \times k}$, $k = 2P$ is a drawn from a normal distribution with unit variance, converging to the Marcenko-Pastur distribution with $q=2$.}
	\label{fig:mplanchist}
\end{figure}
\label{sec:mpstemandhist}
\label{subsec:diagonalapprox}
As a proxy for deep neural network spectra, often the diagonal of the matrix \citep{bishop2006pattern} or the diagonal of a surrogate matrix, such as the Fisher information, or that implied by the values of the Adam Optimizer \citep{chaudhari2016entropy} is used. We plot the true eigenvalue estimates for random matrices pertaining to both the Marcenko-Pastur (Fig. \ref{subfig:MPdiaghist}) and the Wigner density (Fig. \ref{subfig:wignerdiaghist}) in blue, along with the Lanczos estimate in red and the diagonal approximation in yellow. We see here that the diagonal approximation in both cases, fails to adequately the support or accurately model the spectral density, whereas the lanczos estimate is nearly indistinguishable from the true binned eigen-spectrum. This is of-course obvious from the mathematics of the un-normalised Wigner matrix. The diagonal elements are simply draws from the normal distribution $\mathcal{N}(0,1)$ and so we expect the diagonal histogram plot to approximately follow this distribution (with variance $1$). However the second moment of the Wigner Matrix can be given by the Frobenius norm identity 
\begin{equation}
    \mathbb{E}\bigg(\frac{1}{P}\sum_{i}^{P}\lambda_{i}^{2}\bigg) = \mathbb{E}\bigg(\frac{1}{P}\sum_{i,j=1}^{P}\mH_{i.j}^{2}\bigg) = \mathbb{E}\bigg(\frac{1}{P}\chi^{2}_{P^{2}}\bigg) = P
\end{equation}
Similarly for the Marcenko-Pastur distribution, We can easily see that each element of $\mH$ follows a chi-square distribution of $1/T\chi^{2}_{T}$, with mean $1$ and variance $2/T$.
\begin{figure}
	\centering
	\begin{subfigure}{0.45\linewidth}
	\centering
    \includegraphics[trim=2.5cm 8.1cm 2.5cm 8.4cm, clip, width=1\linewidth]{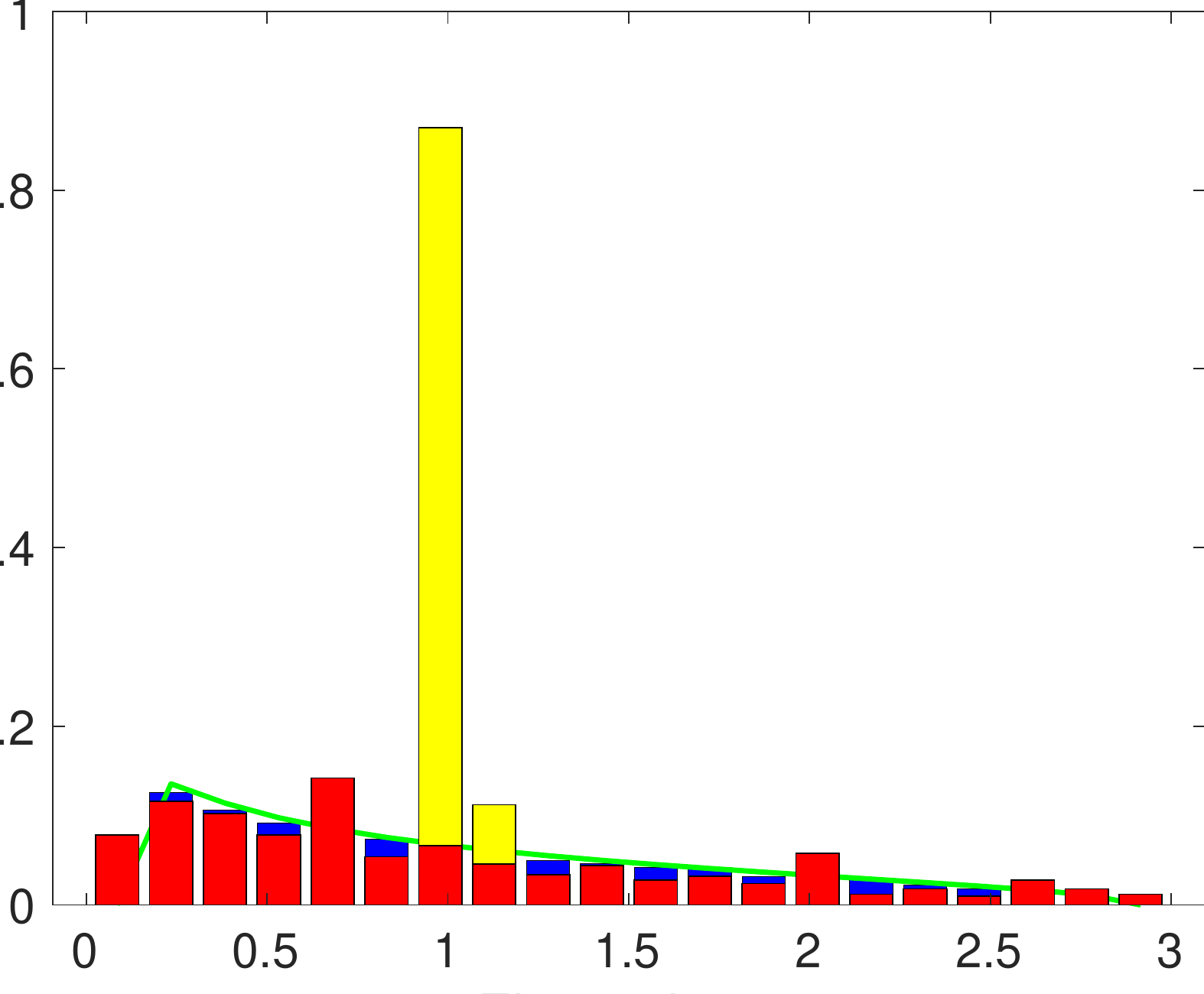}
	\caption{Marcenko-Pastur}
	\label{subfig:MPdiaghist}	
	\end{subfigure}
	\hspace{5pt}
	\begin{subfigure}{0.45\linewidth}
	\centering
	\includegraphics[trim=2.5cm 8.1cm 2.5cm 8.4cm, clip, width=1\linewidth]{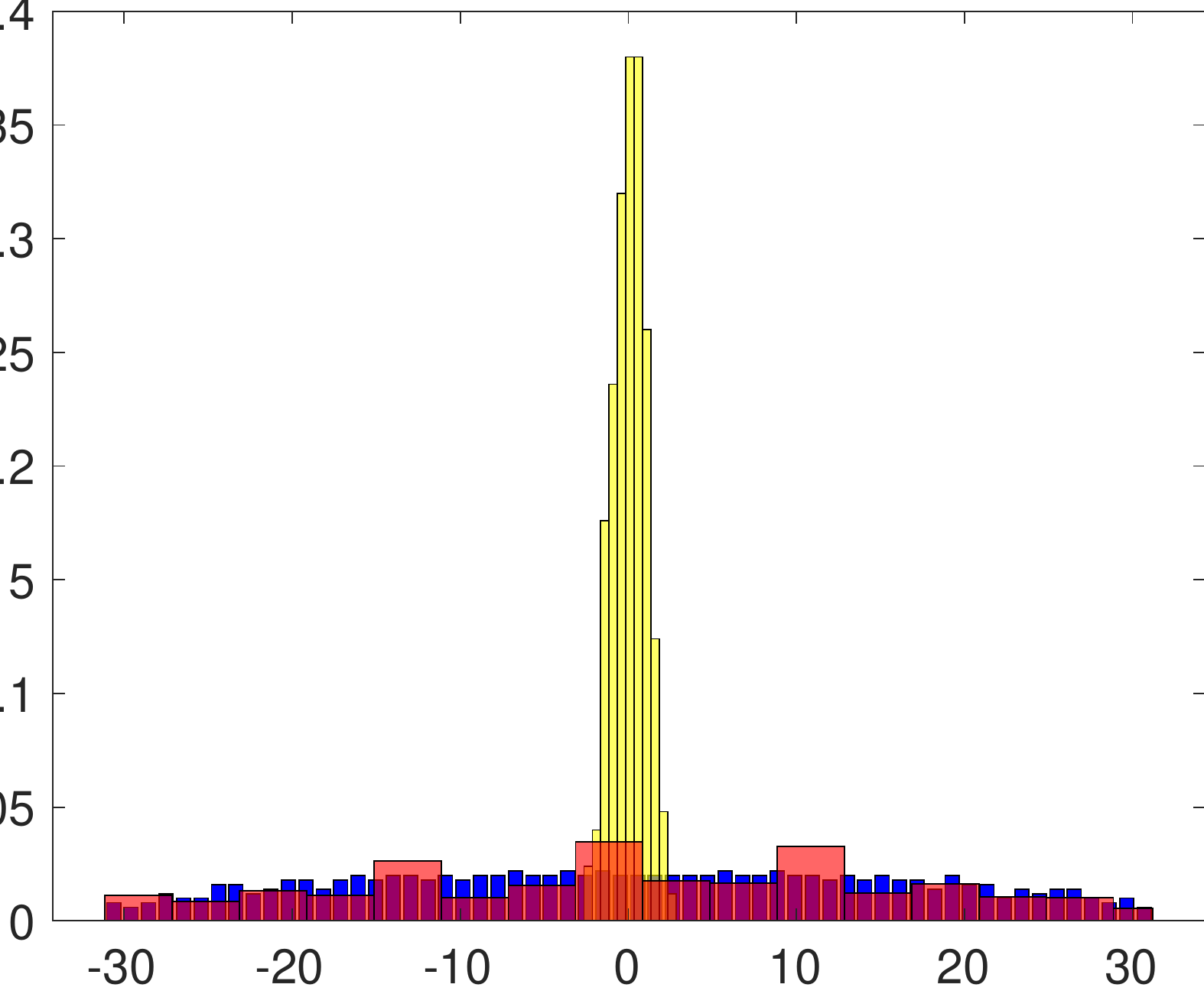}
	\caption{Wigner-Semi Circle}
	\label{subfig:wignerdiaghist}		
	\end{subfigure}
	\caption{Two randomly generated matrices $\mH \in \mathbb{R}^{500\times 500}$ with the histogram of the true eigenvalues in blue, the Lanczos estimate $m=30, d=1$ in red and the diagonal approximation in yellow}
	\label{fig:mpandwignerdiagvslanc}
\end{figure}
\label{sec:lanczoshist}
\subsection{Synthetic Example}

The curvature eigenspectrum of neural network often features a large spike at zero, a right-skewed bulk and some outliers \citep{sagun2016eigenvalues,sagun2017empirical}.\footnote{Some examples of this can be found in later sections on real-life neural network experiments - see Figures \ref{fig:diagggnmcvslanc} and \ref{fig:randomvectors}. } In order to simulate the spectrum of a neural network, we generate a Matrix $\mH \in \mathbb{R}^{1000 \times 1000}$ with $470$ eigenvalues drawn from the uniform distribution from $[0,15]$, $20$ drawn from the uniform $[0,60]$ and $10$ drawn from the uniform $[-10,0]$. The matrix is rotated through a rotation matrix $U$, i.e $\mH = 
\mU\mD\mU^{T}$ where $\mD$ is the diagonal matrix consisting of the eigenvalues and the columns are gaussian random vectors which are orthogonalised using Gram-Schmidt orthogonalisation. The resulting eigenspectrum is given in a histogram in Figure \ref{subfig:madeupmat} and then using the same random vector, successive Lanczos stem plots for different number of iterations $m = [5,30]$ are shown in Figure \ref{fig:madeupmat}. Figure \ref{subfig:madeupmatlanc5}, for a low number of steps, the degeneracy at $\lambda = 0$ is learned, as are the largest and smallest eigenvalues, some information is retained about the bulk density, but some of the outlier eigenvalues around $\lambda \approx 20$ and $ \lambda \approx 30$ are completely missed out, along with all the negative outliers except the largest. 
% This is refined in later estimates with more steps, although we note in Figure \ref{subfig:madeupmatlanc15} that the degeneracy peak at $\lambda=0$ is split into two. 
For $m=30$ even the shape of the bulk is accurately represented, as shown in Figure \ref{subfig:madeupmatlanc30}. Here, we would like to emphasise that learning the outliers is important in the neural network context, as they relate to important properties of the network and the optimisation process \citep{ghorbani2019investigation}.

On the other hand, we note that the diagonal estimate in Figure \ref{subfig:madeupmatdiag} gives absolutely no spectral information, with no outliers shown (maximal and minimal diagonal elements being $5.3$ and $3.3$ respectively and it also gets the spectral mass at $0$ wrong. This builds on section \ref{subsec:diagonalapprox}, as furthering the case against making diagonal approximations in general. In neural networks, the diagonal approximation is similar to positing no correlations between the weights. This is a very harsh assumption and usually a more reasonable assumption is to posit that the correlations between weights in the same layer are larger than between different layers, leading to a block diagonal approximation \citep{Martens2016}, however often when the layers have millions of parameters, full diagonal approximations are still used. \citep{bishop2006pattern,chaudhari2016entropy}.

\section{Neural Network Examples}
\label{sec:neuralnetexamples}
We showcase our spectral learning algorithm and visualization tool on real networks trained on real data-sets and we test on VGG networks \citep{simonyan2014very}. We train our neural networks using stochastic gradient descent with momentum $\rho = 0.9$, using a linearly decaying learning rate schedule. The learning rate at the $t$-th epoch is given by:
\begin{equation}
    \alpha_t = 
    \begin{cases}
      \alpha_0, & \text{if}\ \frac{t}{T} \leq 0.5 \\
      \alpha_0[1 - \frac{(1 - r)(\frac{t}{T} - 0.5)}{0.4}] & \text{if } 0.5 < \frac{t}{T} \leq 0.9 \\
      \alpha_0r, & \text{otherwise}
    \end{cases}
\end{equation}
where $\alpha_0$ is the initial learning rate. $T = 300$ is the total number of epochs budgeted for all experiments. We set $r = 0.01$. We explicitly give an example code run in \ref{sec:dnncoderun}
\begin{figure}
	\centering
	\begin{subfigure}{0.45\linewidth}
	\centering
    \includegraphics[trim=2.5cm 8.1cm 2cm 8.4cm, clip, width=1\linewidth]{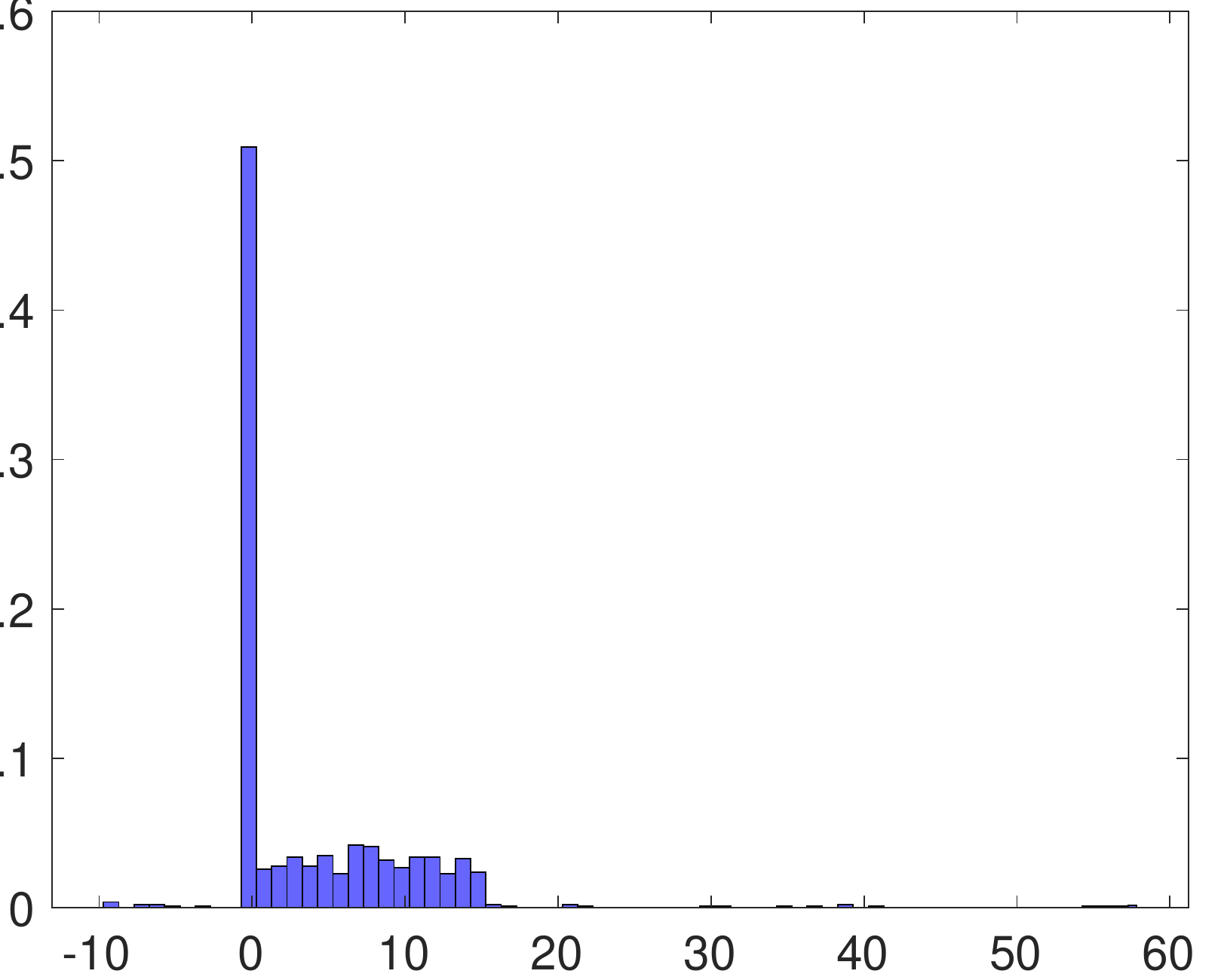}
	\caption{Histogram of $\mH$}
	\label{subfig:madeupmat}	
	\end{subfigure}
	\hspace{5pt}
	\begin{subfigure}{0.45\linewidth}
	\centering
	\includegraphics[trim=2.5cm 8.1cm 2cm 8.4cm, clip, width=1\linewidth]{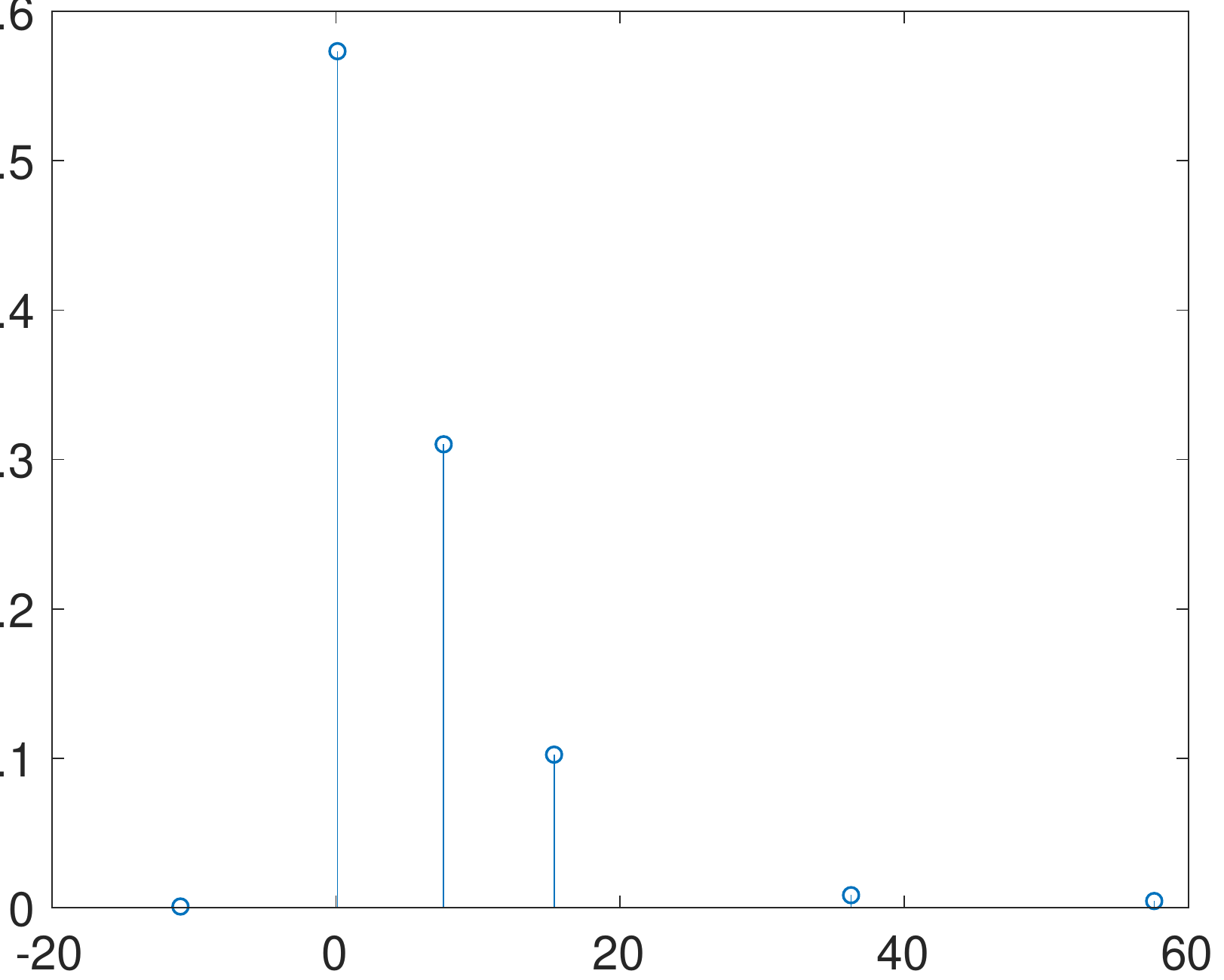}
	\caption{$m=5$ Lanczos Stem}
	\label{subfig:madeupmatlanc5}		
	\end{subfigure}
		\hspace{5pt}
	\begin{subfigure}{0.45\linewidth}
	\centering
	\includegraphics[trim=2.5cm 8.1cm 2cm 8.4cm, clip, width=1\linewidth]{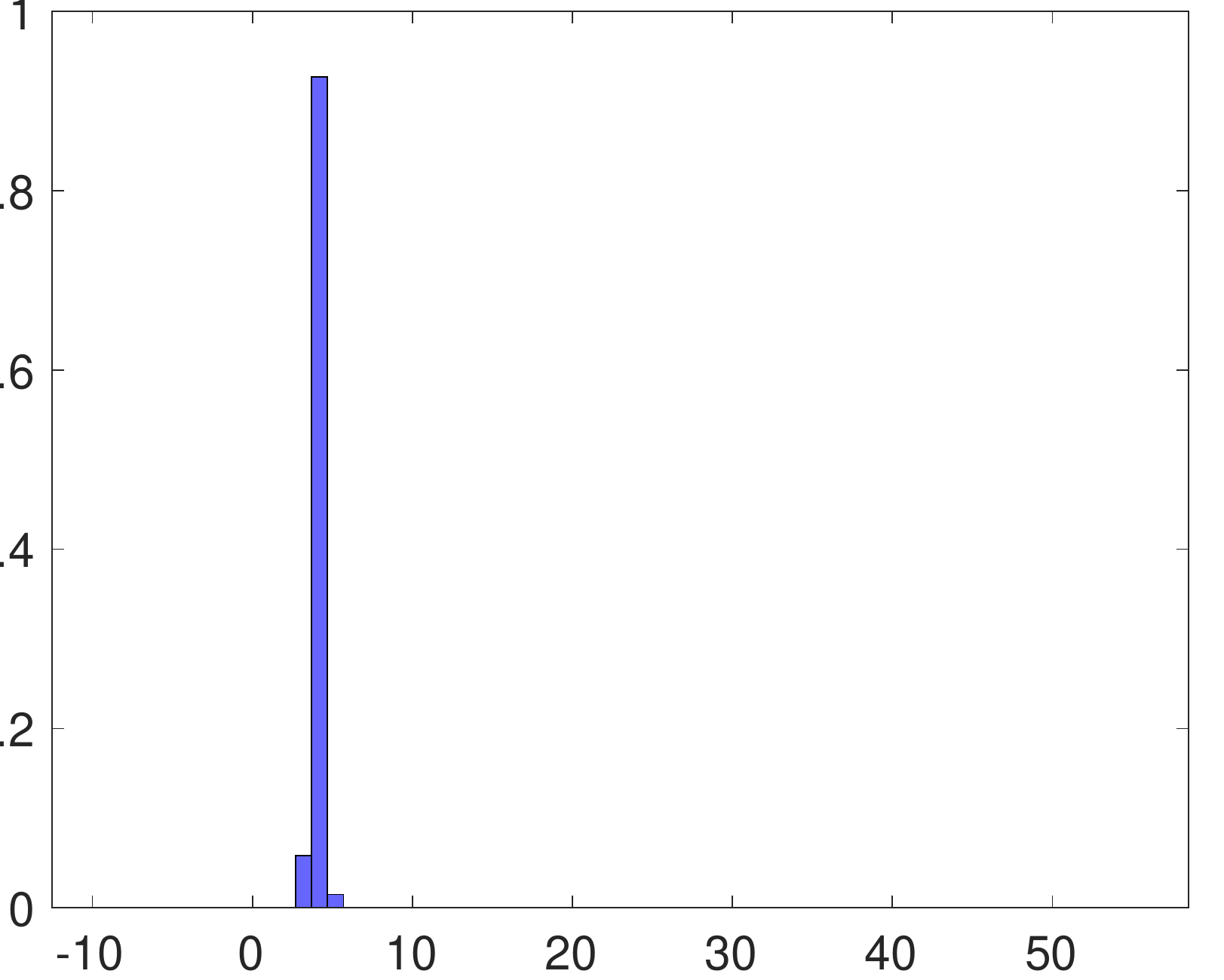}
	\caption{Diagonal of $\mH$}
	\label{subfig:madeupmatdiag}
	\end{subfigure}
		\hspace{5pt}
	\begin{subfigure}{0.45\linewidth}
	\centering
	\includegraphics[trim=2.5cm 8.1cm 2cm 8.4cm, clip, width=1\linewidth]{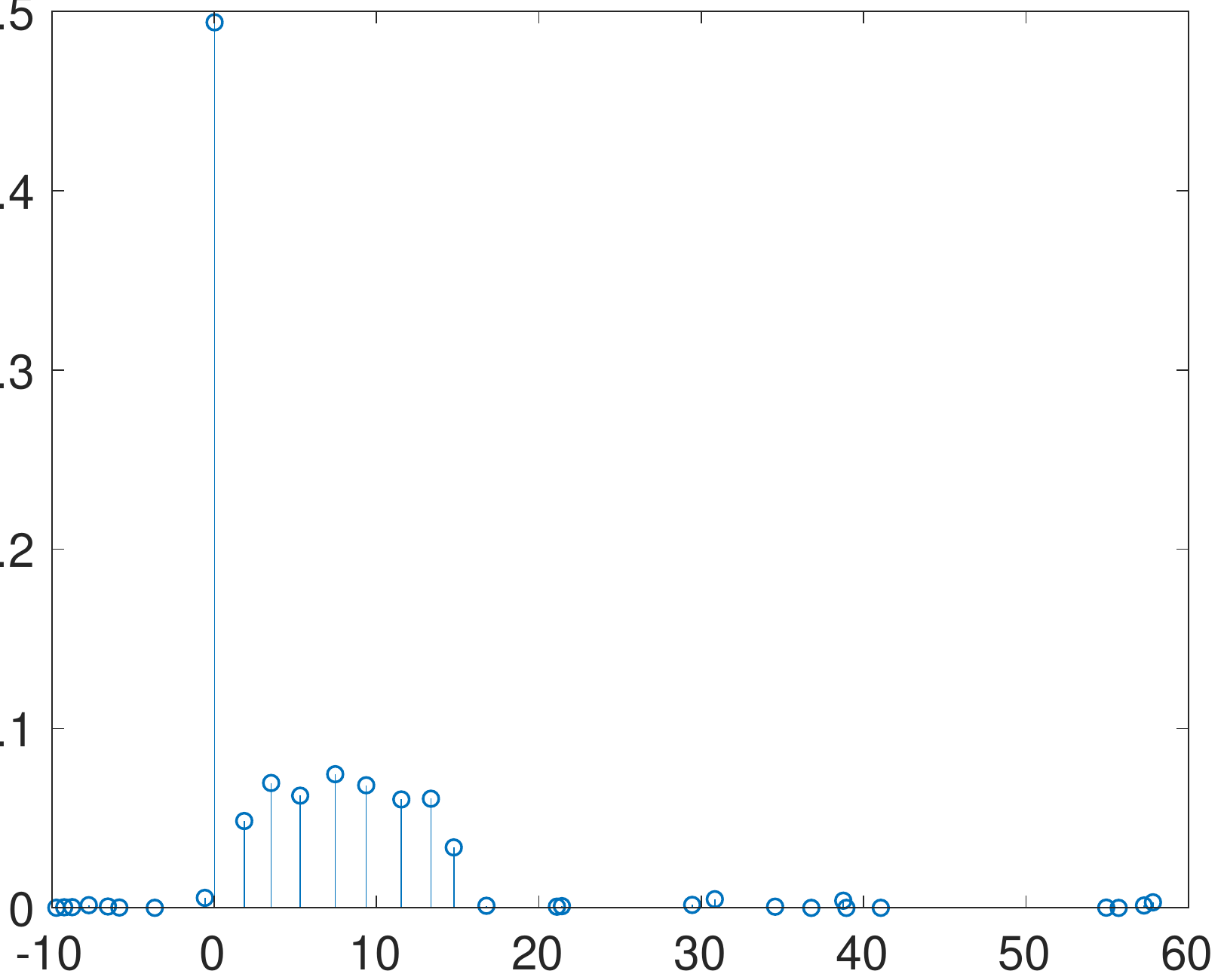}
	\caption{$m=30$ Lanczos Stem}
	\label{subfig:madeupmatlanc30}		
	\end{subfigure}
	\caption{Generated matrices $\mH \in \mathbb{R}^{1000\times 1000}$ with known eigenspectrum and Lanczos stem plots for different values of $m = \{5,15,30\}$}
	\label{fig:madeupmat}
\end{figure}
We compare our method against recently developed open-source tools which calculate on the fly diagonal Hessian and Generalised Gauss-Newton diagonal approximations \citep{dangel2019backpack}.

\subsection{VGG-$16$ CIFAR-100 Dataset}
\label{subsec:vgg16c100}
We train a 16-layer VGG network, comprising of $P=15,291,300$ parameters on the CIFAR-$100$ dataset, using $\alpha_{0}=1$. Even for this relatively small model, the open-source Hessian and GGN exact diagonal computations require over $125$GB of GPU memory and so to avoid re-implementing the library to support multiple GPUs and node communication we use the Monte Carlo approximation to the GGN diagonal against both our GGN-Lanczos and Hessian-Lanczos spectral visualizations. We plot a histogram of the Monte Carlo approximation of the diagonal GGN (Diag-GGN) against both the Lanczos GGN (Lanc-GGN) and Lanczos Hessian (Lanc-Hess) in Figure \ref{fig:diagggnmcvslanc}. Note that as the Lanc-GGN and Lanc-Hess are displayed as stem plots (with the discrete spectral density summing to $1$ as opposed to the histogram area summing to $1$). 
\newline
\begin{figure}
	\centering
	\begin{subfigure}{0.45\linewidth}
	\centering
    \includegraphics[trim=0cm 0cm 0cm 0cm, clip, width=1\linewidth]{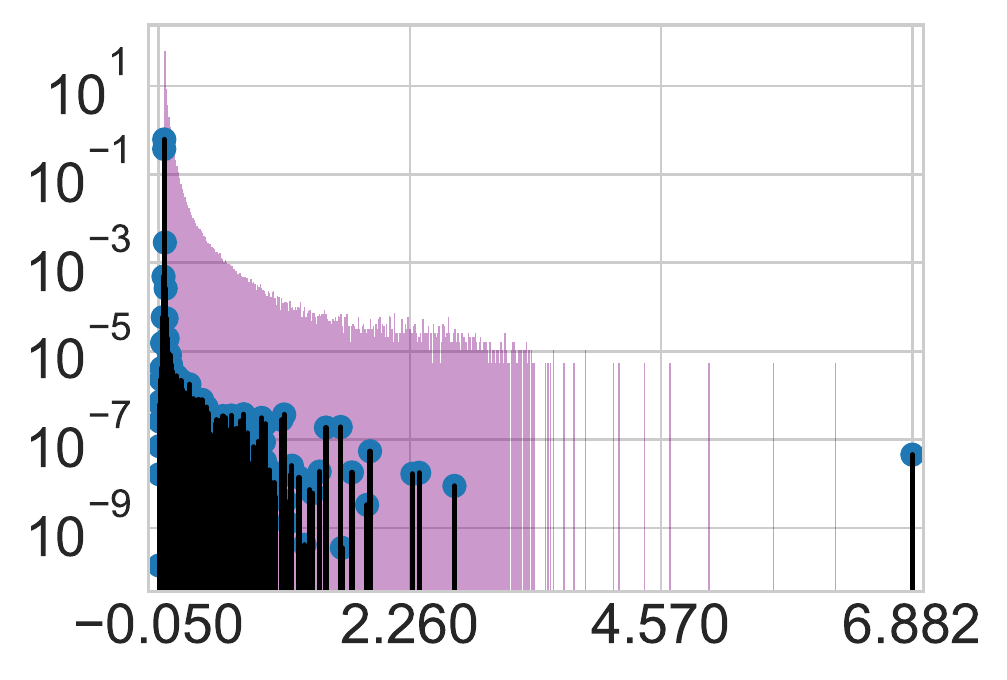}
	\caption{Diag-GGN \& Lanc-GGN}
	\label{subfig:diagggnmcvslancggn}	
	\end{subfigure}
	\hspace{5pt}
	\begin{subfigure}{0.45\linewidth}
	\centering
	\includegraphics[trim=0cm 0cm 0cm 0cm, clip, width=1\linewidth]{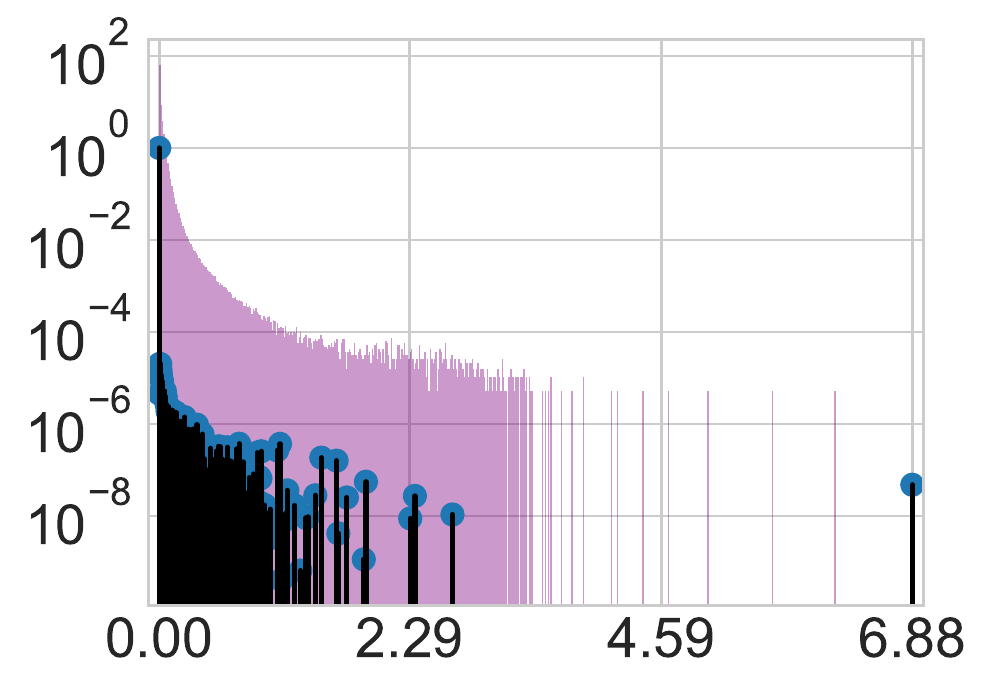}
	\caption{Diag-GGN \& Lanc-Hess}
	\label{subfig:diagggnmcvslanchess}		
	\end{subfigure}
	\caption{Diagonal Generalised Gauss Newton monte carlo approximation (Diag-GGN) against $m=100$ Lanczos using Gauss-Newton vector products (Lanc-GGN) or Hessian vector products (Lanc-Hess)}
	\label{fig:diagggnmcvslanc}
\end{figure}
\label{sec:lanczoshist}

We note that the Gauss-Newton approximation quite closely resembles its Hessian counterpart, capturing the majority of the bulk and the outlier eigenvectors at $\lambda_{1} \approx 6.88$ and the triad near $\lambda_{i} \approx 2.29$. The Hessian does still have significant spectral mass on the negative axis, around $37\%$. However most of this is captured by a Ritz value at $-0.0003$, with this removed, the negative spectral mass is only $0.05\%$. However as expected from our previous section, the Diag-GGN gives a very poor spectral approximation. It vastly overestimates the bulk region, which extends well beyond $\lambda \approx 1$ implied by Lanczos and adds many spurious outliers between $3$ and the misses the largest outlier of $6.88$. 

\paragraph{Computational Cost} Using a single NVIDIA GeForce GTX 1080 Ti GPU, the Gauss-Newton takes an average $26.5$ seconds for each Lanczos iteration with the memory useage $2850$Mb. Using the Hessian takes an average of $27.9$ seconds for each Lanczos iteration with $2450$Mb memory usage.

\section{Effect of Varying Random Vectors}
\label{sec:highdimensionste}
Given that the proofs for the moments of Lanczos matching those of the underlying spectral density, are true over the expectation over the set of random vectors and in practice we only use a Monte Carlo average of random vectors, or in our experiments using stem plots, just a single random vector. We justify this with the following Lemma
\begin{lemma}
Let $\vu \in \mathbb{R}^{P\times 1}$ random vector, where $\vu_{i}$ is zero mean and unit variance and finite $4$'th moment $\mathbb{E}[\vu_{i}^{4}] = m_{4}$. Then for $\mH \in \mathbb{R}^{P\times P}$, then \newline
\begin{equation}
\nonumber
    \begin{aligned}
        & i) \mathbb{E}[\vu^{T}\mH\vu] = \Tr \mH \\
        & ii) \Var [\vu^{T}\mH\vu] \leq (2+m_{4}) \Tr (\mH^{T}\mH) \\
    \end{aligned}
\end{equation}

\end{lemma}
\begin{proof}
\begin{equation}
\mathbb{E}[\vu^{T}\mH\vu] = \sum_{i,j=1}^{P}\mH_{i,j}\mathbb{E}[\vu_{i}\vv_{j}]= \sum_{i=1}^{P}\mH_{i,i} = \Tr \mH
\end{equation}
\begin{equation}
\begin{aligned}
   &  \mathbb{E}[||\vu^{T}\mH\vu||^{2}]   = \sum_{i,j}\sum_{k,l}\mH_{i,j}\mH_{k,l}^{T}\mathbb{E}[\vu_{i}\vu_{j}^{T}\vu_{k}\vu_{l}^{T}] \\
     &  \sum_{i,j}\sum_{k,l}\mH_{i,j}\mH_{k,l}^{T}[\delta_{i,j}\delta_{k,l}+\delta_{i,l}\delta_{j,k}+\delta_{i,k}\delta_{j,l}+m_{4}\delta_{i,j,k,l}]  \\
    & = (\Tr \mH)^{2}+(2+m_{4})\Tr(\mH^{2}) \\
\end{aligned}
\end{equation}
\end{proof}
\begin{remark}
Let us consider the signal to noise ratio for some positive definite $\mH \succ c\mI$
\begin{equation}
    \begin{aligned}
         & \bigg(\frac{\sqrt{\Var [\vu^{T}\mH\vu]}}{\mathbb{E}[\vu^{T}\mH\vu]}\bigg)^{2} \propto \frac{1}{1+\frac{\sum_{i\neq j}^{P}\lambda_{i}\lambda_{j}}{\sum_{k}^{P}\lambda_{k}^{2}}} = \frac{1}{1+\frac{P-1\langle\lambda_{i}\lambda_{j}\rangle}{\langle\lambda_{k}^{2}\rangle}} \\
         & \leq \frac{1}{1+\frac{P-1}{\kappa^{2}}}
    \end{aligned}
\end{equation}
where $\langle .. \rangle$ denotes the arithmetic average. For the extreme case of all eigenvalues being identical, the condition number $\kappa = 1$ and hence this reduces to $1/P \rightarrow 0$ in the $P \rightarrow \infty$ limit, whereas for a rank-$1$ matrix, this ratio remains $1$. For the MP density, which well models neural network spectra, $\kappa$ is not a function of $P$ as $P \rightarrow \infty$ and hence we also expect this benign dimensional scaling to apply.
\end{remark}
We verify this high dimensional result experimentally, by running the same spectral visualisation as in Section \ref{subsec:vgg16c100} but using two different random vectors. We plot the results in Figure \ref{fig:randomvectors}. We find both figures \ref{subfig:vgg16vec1} \& \ref{subfig:vgg16vec2} to be close to visually indistinguishable. There are minimal differences in the extremal eigenvalues, with former giving $\{\lambda_{1},\lambda_{n}\} =$
$ \{6.8885,-0.0455\}$ and the latter $ \{6.8891,-0.0456\}$, but the degeneracy at $0$, bulk, triplet of outliers at $2.27$ and the large outlier at $6.89$ is unchanged. We include the code to run in \ref{subsec:example}

\begin{figure}[h!]
	\centering
	\begin{subfigure}{0.45\linewidth}
	\centering
    \includegraphics[trim=0cm 0cm 0cm 0cm, clip, width=1\linewidth]{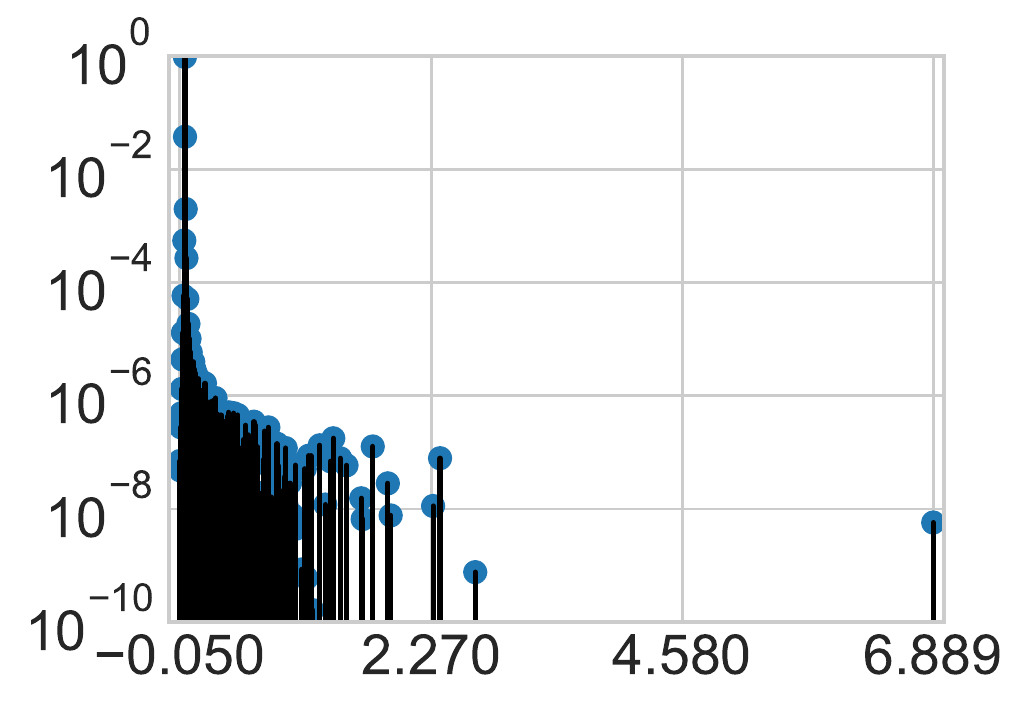}
	\caption{Vector $2$}
	\label{subfig:vgg16vec1}	
	\end{subfigure}
	\hspace{5pt}
	\begin{subfigure}{0.45\linewidth}
	\centering
	\includegraphics[trim=0cm 0cm 0cm 0cm, clip, width=1\linewidth]{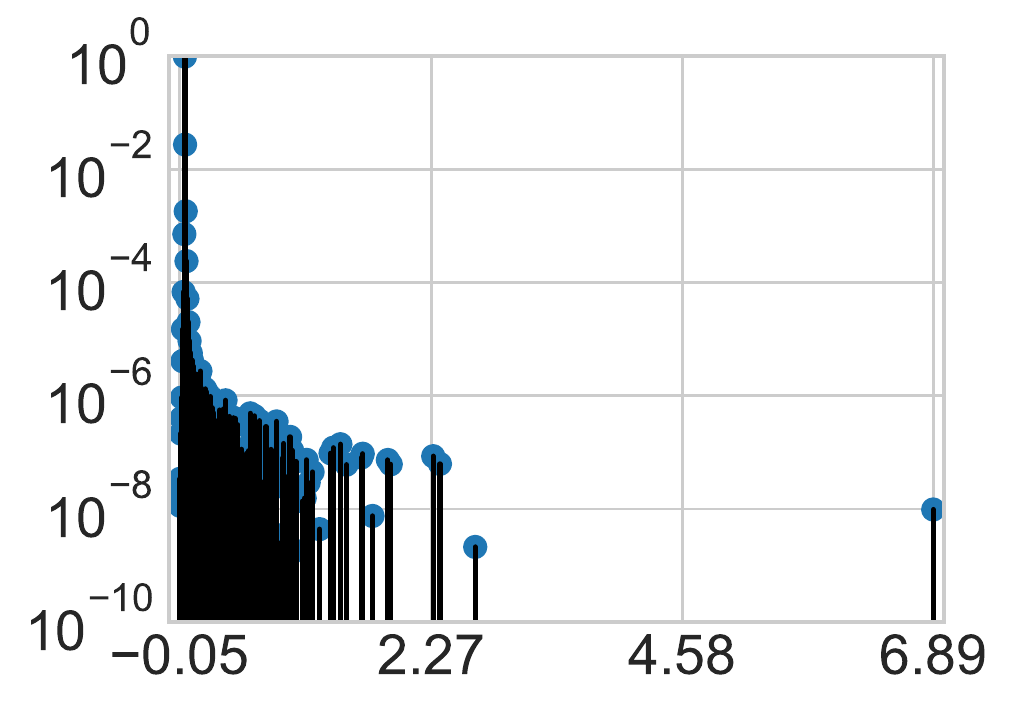}
	\caption{Vector $3$}
	\label{subfig:vgg16vec2}		
	\end{subfigure}
	\caption{VGG$16$ Epoch $300$ end of training Lanczos stem plot for different random vectors}
	\label{fig:randomvectors}
\end{figure}
\label{sec:lanczoshist}

\subsection{Why we don't kernel smooth}
\label{subsec:kernelsmoothingbad}
Concurrent work, which has also used Lanczos with the Pearlmutter trick to learn the Hessian \citep{yao2018hessian,ghorbani2019investigation}, typically uses $n_{v}$ random vectors and then uses kernel smoothing, to give a final density. In this section we argue that beyond costing a factor of $n_{v}$ more computationally, that the extra compute extended in order to get more accurate moment estimates, which we already argued in Section \ref{sec:highdimensionste} are asymptotically error free, is wasted due to the kernel smoothing \citep{granziol2019meme}.
The smoothed spectral density takes the form:
\begin{equation} 
\begin{aligned}
   & \tilde{p}(\lambda) = \int k_\sigma(\lambda-\lambda') p(\lambda') d\lambda'  = \sum_{i=1}^{n}w_{i}k_\sigma(\lambda-\lambda_{i})
\end{aligned}
 \label{eq:smootheddensity}
\end{equation}
We make some assumptions regarding the nature of the kernel function, $k_\sigma(\lambda-\lambda_{i})$, in order to prove our main theoretical result about the effect of kernel smoothing on the moments of the underlying spectral density. Both of our assumptions are met by (the commonly employed) Gaussian kernel.

%However, this introduces hyperparameters, such as the choice of convolving kernel, the smoothing parameter or the number of bins for the histogram, which heavily affect the resolution of the spectra. 
\begin{assumption}
\label{ass:infdomain}
The kernel function $k_\sigma(\lambda-\lambda_{i})$ is supported on the real line $[-\infty,\infty]$.
\end{assumption}
\begin{assumption}
\label{ass:sym}
The kernel function $k_\sigma(\lambda-\lambda_{i})$ is symmetric and permits all moments.
\end{assumption}
\begin{theorem}
The $m$-th moment of a Dirac mixture $\sum_{i=1}^{n}w_{i}\delta(\lambda-\lambda_{i})$, which is smoothed by a kernel $k_{\sigma}$ satisfying assumptions \ref{ass:infdomain} and \& \ref{ass:sym}, is perturbed from its unsmoothed counterpart by an amount $\sum_{i=1}^{n}w_{i} \sum_{j=1}^{r/2} {r \choose 2j}\mathbb{E}_{k_\sigma(\lambda)}(\lambda^{2j})\lambda_{i}^{m-2j}$, where $r=m$ if $m$ is even and $m-1$ otherwise. $\mathbb{E}_{k_\sigma(\lambda)}(\lambda^{2j})$ denotes the $2j$-th central moment of the kernel function $k_\sigma(\lambda)$.
\end{theorem}
\begin{proof}
The moments of the Dirac mixture are given as, 
\begin{equation}
    \langle \lambda^{m} \rangle = \sum_{i=1}^{n}w_{i}\int \delta(\lambda-\lambda_{i})\lambda^{m}d\lambda = \sum_{i=1}^{n}w_{i}\lambda_{i}^{m}.
\end{equation}
The moments of the modified smooth function (Equation \eqref{eq:smootheddensity}) are
\begin{equation}
    \begin{aligned}
    \langle \tilde{\lambda}^{m} \rangle & =\sum_{i=1}^nw_{i}\int k_\sigma(\lambda-\lambda_{i})\lambda^{m}d\lambda \\& 
    = \sum_{i=1}^nw_{i}\int k_\sigma(\lambda')(\lambda'+\lambda_{i})^{m}d\lambda' \\ & = \langle \lambda^{m} \rangle + \sum_{i=1}^{n}w_{i} \sum_{j=1}^{r/2} {r \choose 2j}\mathbb{E}_{k_\sigma(\lambda)}(\lambda^{2j})\lambda_{i}^{m-2j}.
    \end{aligned}
\end{equation}
We have used the binomial expansion and the fact that the infinite domain is invariant under shift reparametarization and the odd moments of a symmetric distribution are $0$. 
\end{proof}

\begin{remark}
The above proves that kernel smoothing alters moment information, and that this process becomes more pronounced for higher moments. Furthermore, given that $w_{i} > 0$, $\mathbb{E}_{k_\sigma(\lambda)}(\lambda^{2j}) > 0$ and (for the GGN $lambda_{i} > 0$, the corrective term is manifestly positive, so the smoothed moment estimates are biased.
\end{remark}
% \subsubsection{Guarantees for the Diagonal Approximation}
\section{Local loss landscape} The Lanczos algorithm with enforced orthogonality initialised with a random vector gives a moment matched discrete approximation to the Hessian spectrum. However this information is local to the point in weight space $\vw$ and the quadratic approximation may break down within the near vicinity. To investigate this, we use the \textbf{loss landscape visualisation} function of our package: We display this for the VGG-$16$ on CIFAR-$100$ in Figure \ref{fig:vgg16c100losssurface}. We see for the training loss \ref{subfig:vgg16c100trainlosssurface} that the eigenvector corresponding to the largest eigenvalue $\lambda = 6.88$ only very locally corresponds to the sharpest increase in loss for the training, with other extremal eigenvectors, corresponding to the eigenvalues $\lambda = \{2.67,2.35\}$ overtaking it in loss change relatively rapidly. Interestingly for the testing loss, all the extremal eigenvectors change the loss much more rapidly, contradicting previous assertions that the test loss is a ''shifted" version of the training loss \citep{he2019asymmetric,izmailov2018averaging}. We do however note some small asymettry between the changes in loss along the opposite ends of the eigenvectors. The flat directions remain flat locally and some of the eigen-vectors corresponding to negative values correspond to decreases in test loss. We include the code in \ref{subsec:example2}

\begin{figure}[h!]
	\centering
	\begin{subfigure}{0.34\linewidth}
	\centering
    \includegraphics[trim=0cm 0cm 0cm 0cm, clip, width=1\linewidth]{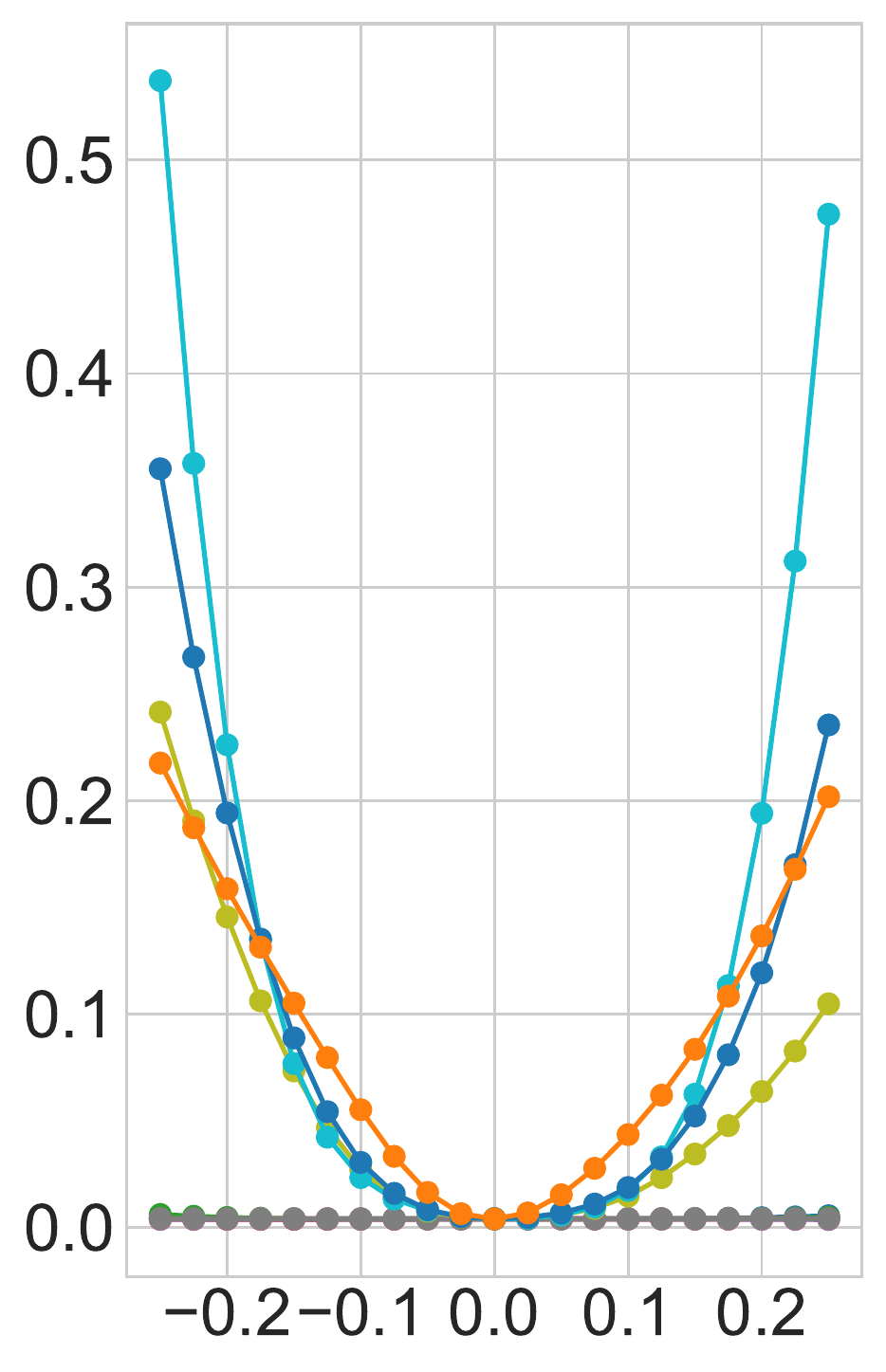}
	\caption{Training Loss}
	\label{subfig:vgg16c100trainlosssurface}	
	\end{subfigure}
	\hspace{5pt}
	\begin{subfigure}{0.56\linewidth}
	\centering
	\includegraphics[trim=0cm 0cm 0cm 0cm, clip, width=1\linewidth]{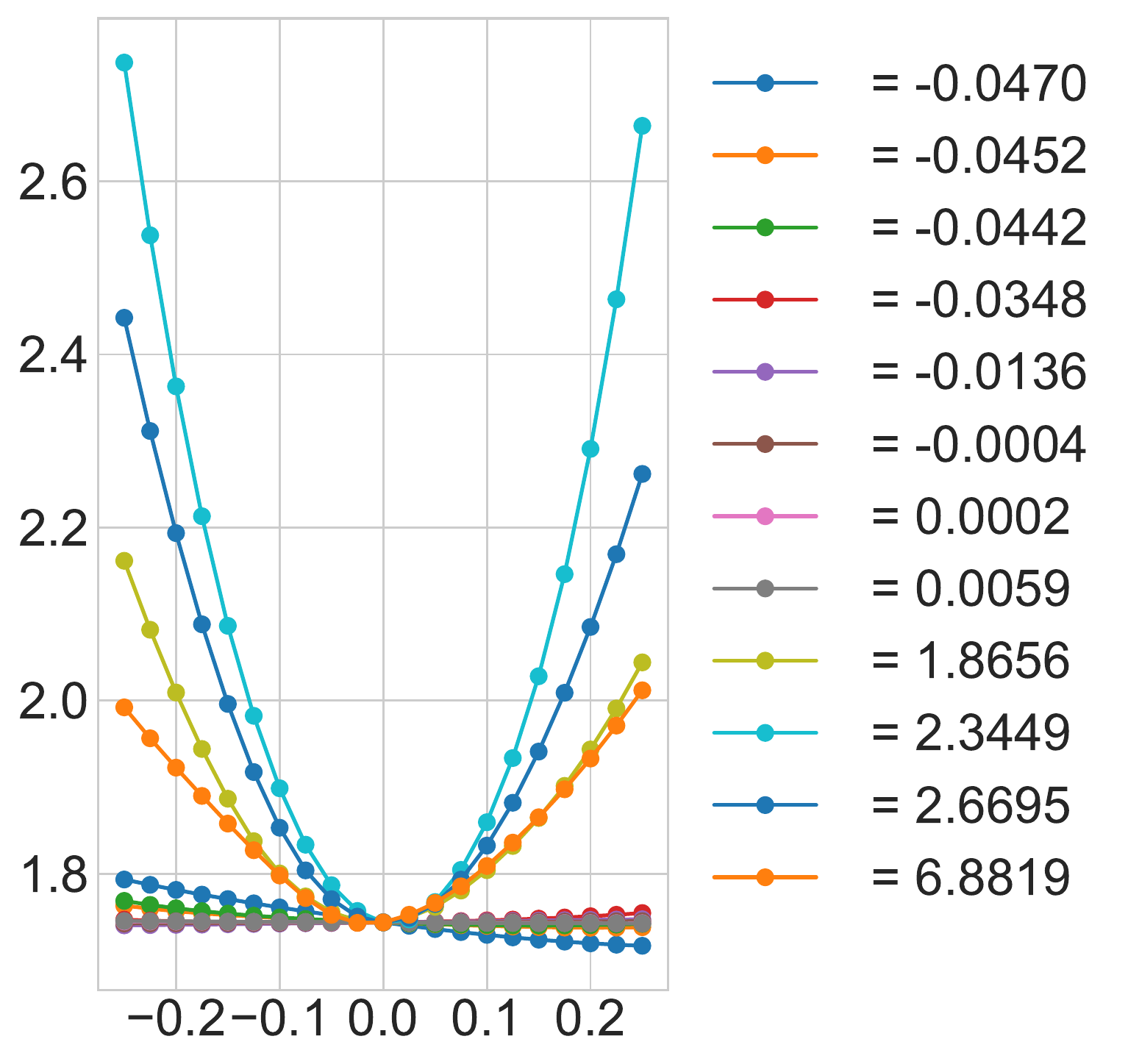}
	\caption{Testing Loss}
	\label{subfig:vgg16c100testlosssurface}		
	\end{subfigure}
	\caption{VGG-$16$ CIFAR-$100$ Loss surface visualised along $6$ negative and position eigenvalues}
	\label{fig:vgg16c100losssurface}
\end{figure}

\begin{figure}[h!]
	\centering
	\begin{subfigure}{0.33\linewidth}
	\centering
    \includegraphics[trim=0cm 0cm 0cm 0cm, clip, width=1\linewidth]{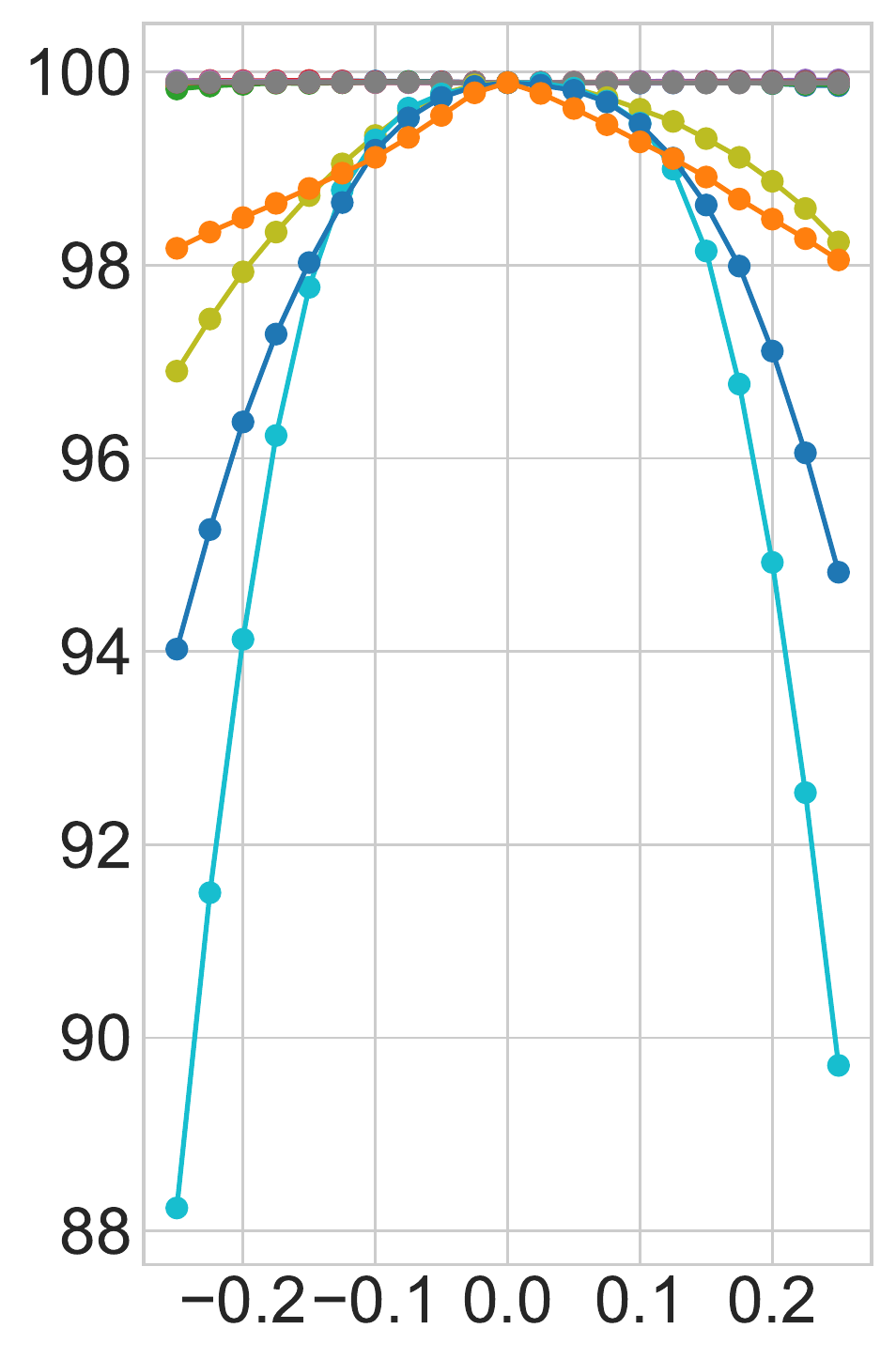}
	\caption{Training Accuracy}
	\label{subfig:vgg16c100trainaccsurface}	
	\end{subfigure}
	\hspace{5pt}
	\begin{subfigure}{0.57\linewidth}
	\centering
	\includegraphics[trim=0cm 0cm 0cm 0cm, clip, width=1\linewidth]{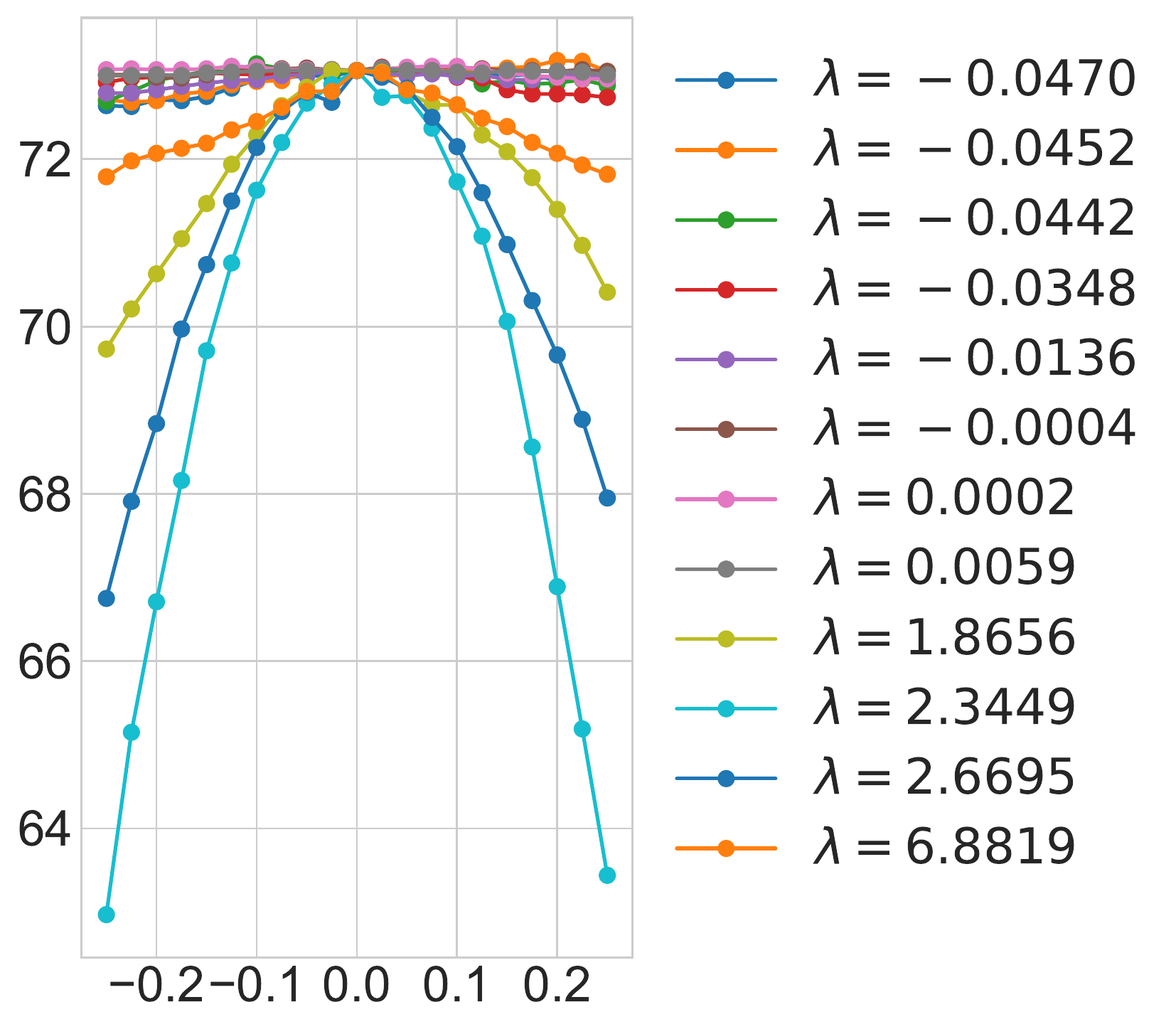}
	\caption{Testing Accuracy}
	\label{subfig:vgg16c100testaccsurface}		
	\end{subfigure}
	\caption{VGG-$16$ CIFAR-$100$ Accuracy surface visualised along $6$ negative and position eigenvalues}
	\label{fig:vgg16c100accsurface}
\end{figure}

\subsubsection*{CIFAR-$10$ Dataset}
\label{subsec:vgg16c10}
\begin{figure}[h!]
	\centering
	\begin{subfigure}{0.34\linewidth}
	\centering
    \includegraphics[trim=0cm 0cm 0cm 0cm, clip, width=1\linewidth]{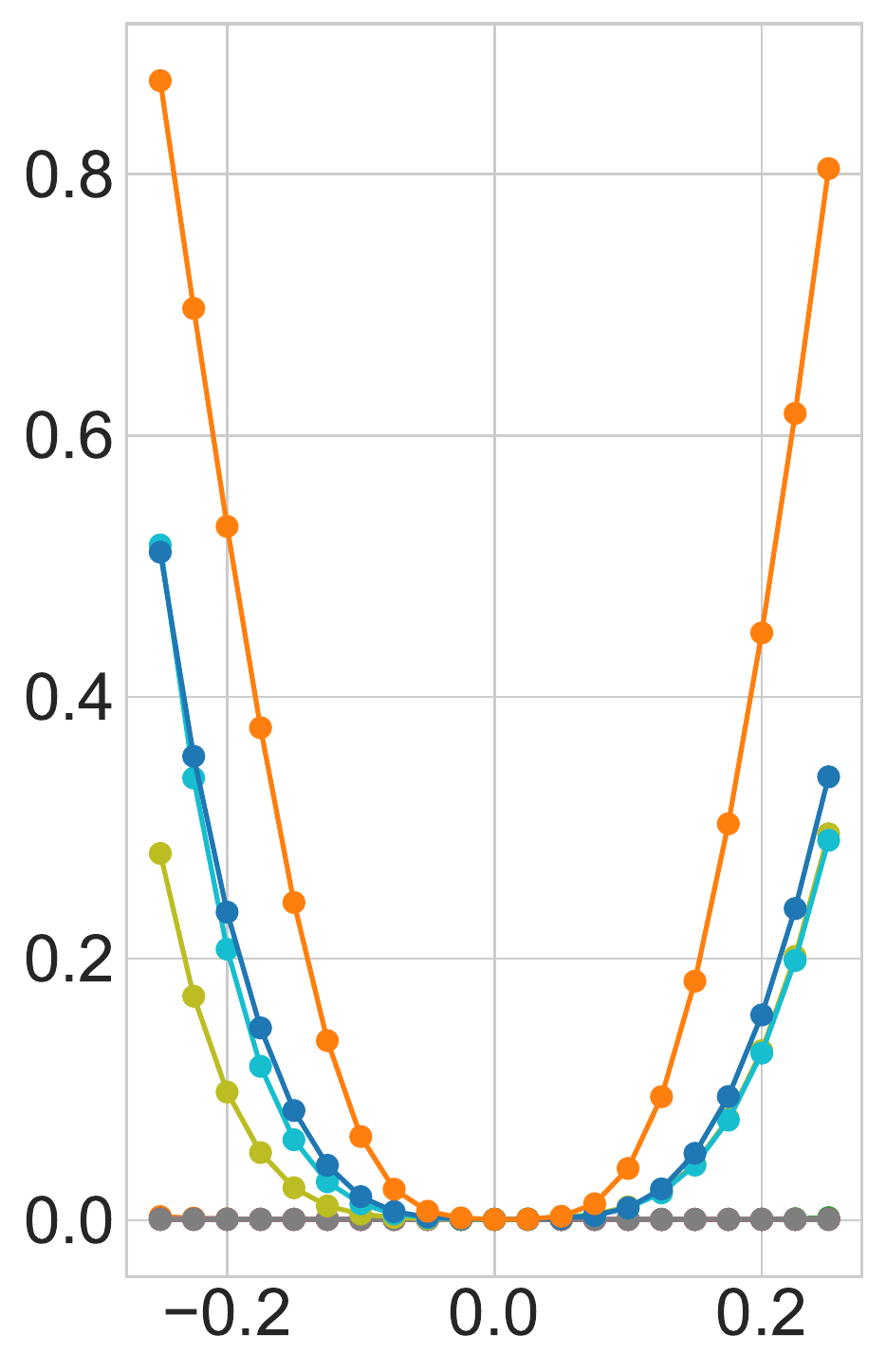}
	\caption{Training Loss}
	\label{subfig:vgg16c10trainlosssurface}	
	\end{subfigure}
	\hspace{5pt}
	\begin{subfigure}{0.56\linewidth}
	\centering
	\includegraphics[trim=0cm 0cm 0cm 0cm, clip, width=1\linewidth]{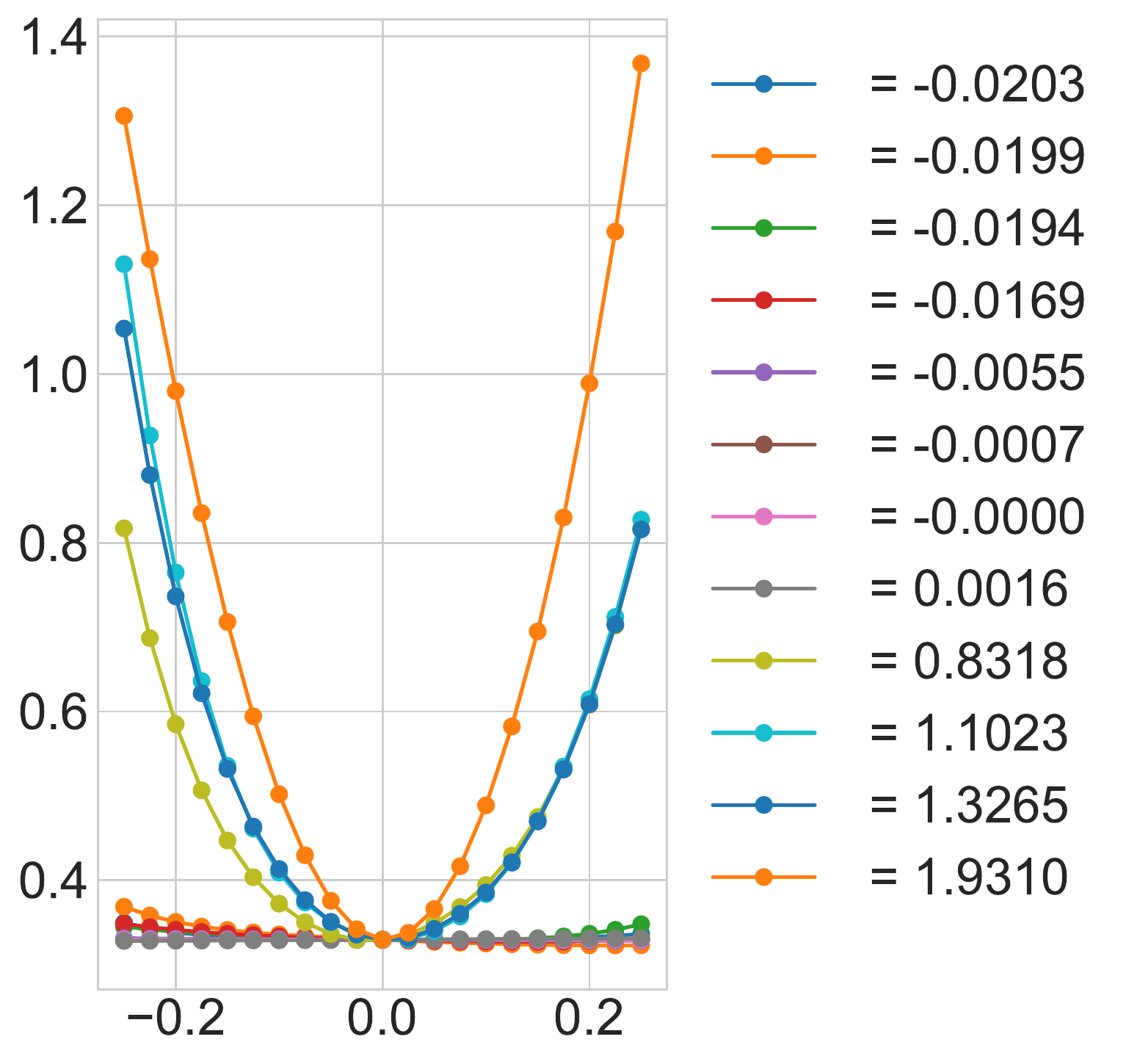}
	\caption{Testing Loss}
	\label{subfig:vgg16c10testlosssurface}		
	\end{subfigure}
	\caption{VGG-$16$ CIFAR-$10$ Loss surface visualised along $6$ negative and position eigenvalues}
	\label{fig:vgg16c100losssurface}
\end{figure}

To showcase the ability of our software to handle multiple datasets we display the Hessian of the VGG-$16$ trained in an identical fashion as its CIFAR-$100$ counterpart of CIFAR-$10$ in Figure \ref{fig:cifar10hessian}, along with the a plot of a selection of Ritz vectors traversing the training loss surface in Figure \ref{subfig:vgg16c10trainlosssurface} and testing loss surface in Figure \ref{subfig:vgg16c10testlosssurface} along with also the training accuracy surface (Figure \ref{subfig:vgg16c10trainaccsurface}) and testing accuracy surface (Figure \ref{subfig:vgg16c10testaccsurface}).

\begin{figure}[h!]
	\centering
	\begin{subfigure}{0.33\linewidth}
	\centering
    \includegraphics[trim=0cm 0cm 0cm 0cm, clip, width=1\linewidth]{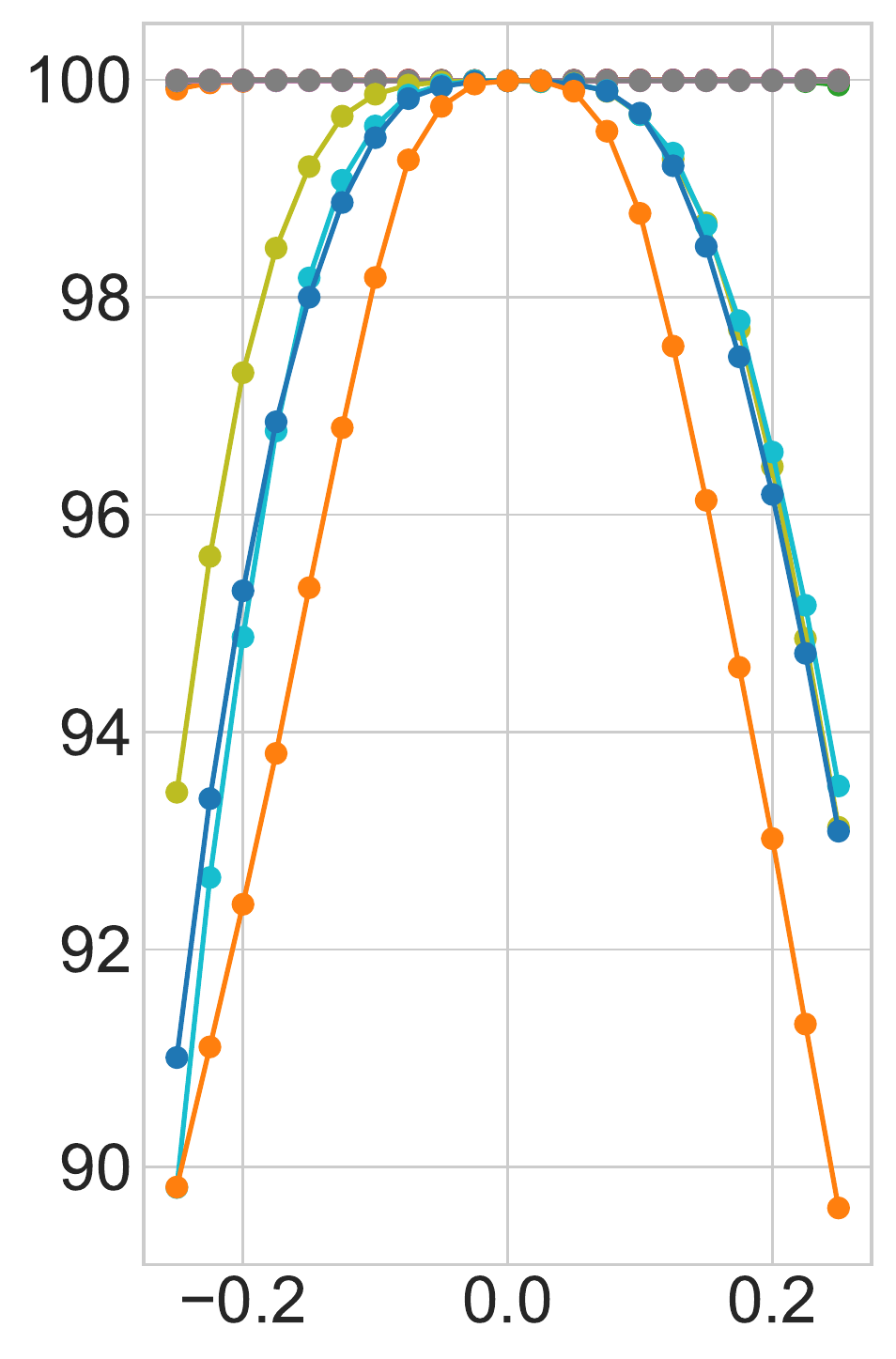}
	\caption{Training Accuracy}
	\label{subfig:vgg16c10trainaccsurface}	
	\end{subfigure}
	\hspace{5pt}
	\begin{subfigure}{0.57\linewidth}
	\centering
	\includegraphics[trim=0cm 0cm 0cm 0cm, clip, width=1\linewidth]{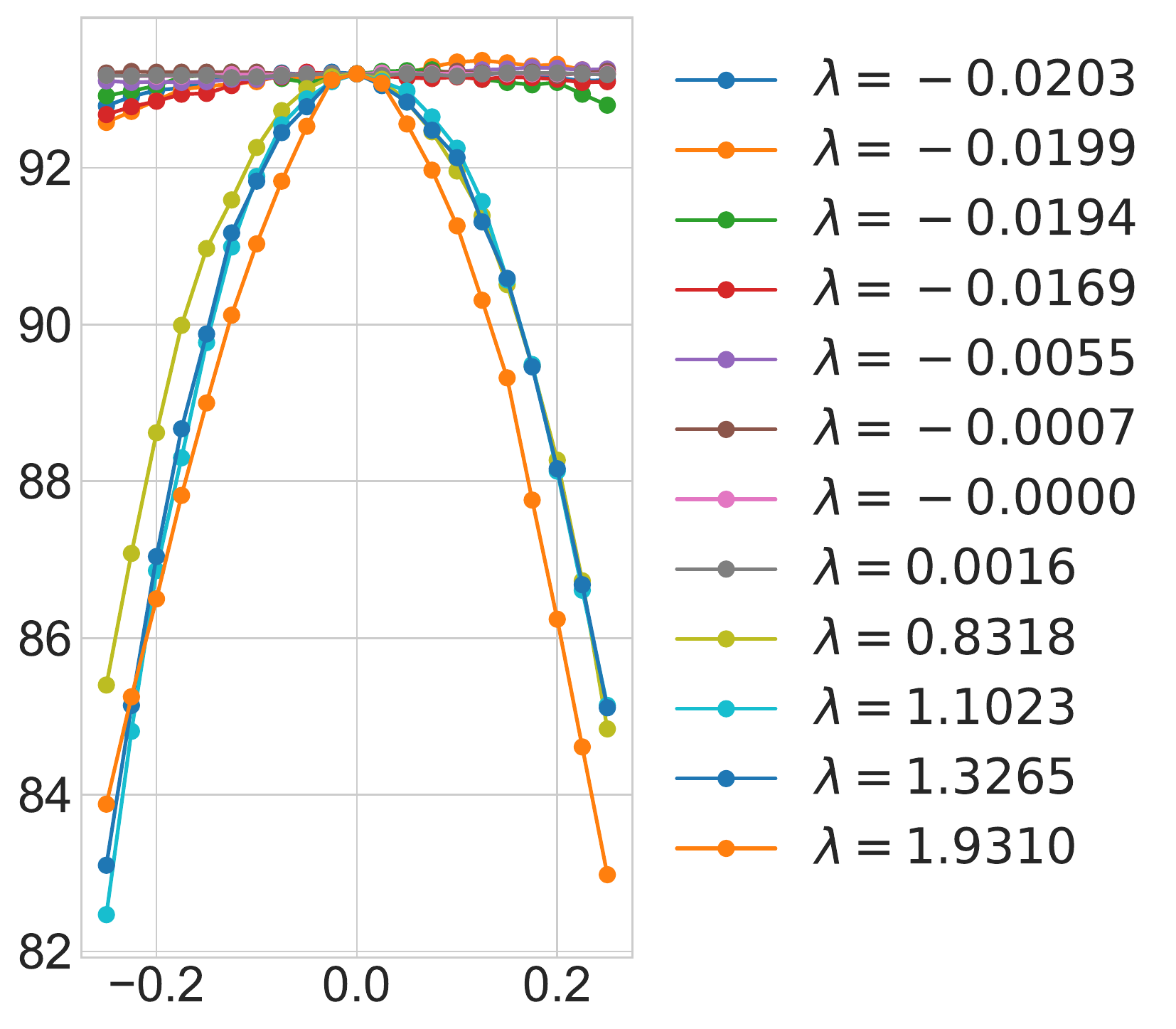}
	\caption{Testing Accuracy}
	\label{subfig:vgg16c10testaccsurface}		
	\end{subfigure}
	\caption{VGG-$16$ CIFAR-$10$ Accuracy surface visualised along negative and positive eigenvalues}
	\label{fig:vgg16c10accsurface}
\end{figure}
\begin{figure}[h!]
    \centering
    \includegraphics[trim={0 0cm 0 0cm},clip, width=0.5\textwidth]{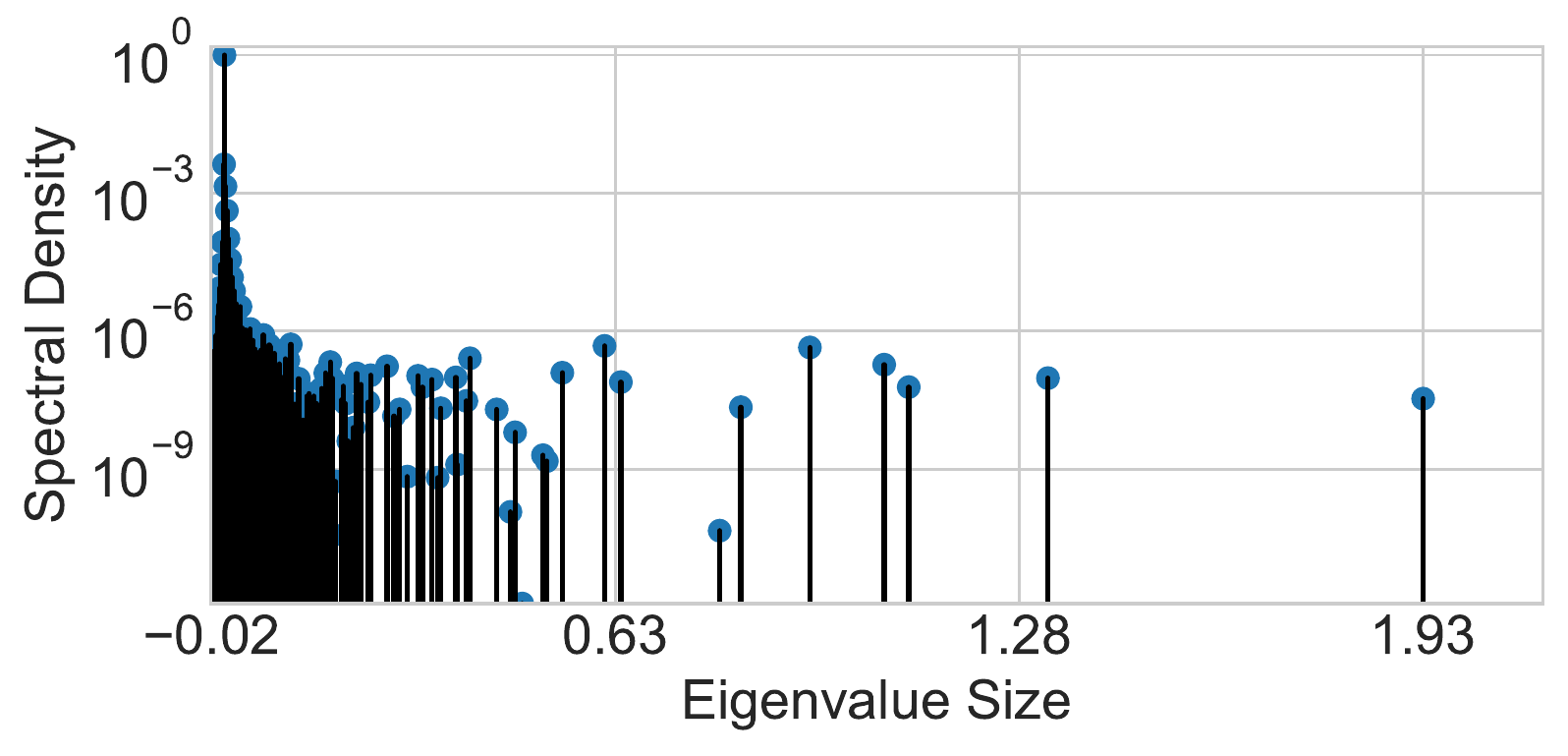}
    \caption{VGG-$16$ CIFAR-$10$ epoch $300$ Hessian}
    \label{fig:cifar10hessian}
\end{figure}

\section{Conclusion}
We introduce the \textbf{Deep Curvature} suite in \textbf{PyTorch} framework, based on the Lanczos algorithm implemented in GPyTorch \citep{gardner2018gpytorch}, that allows deep learning practitioners to learn spectral density information as well as eigenvalue/eigenvector pairs of the curvature matrices at specific points in weight space. Together with the software, we also include a succinct summary of the linear algebra, iterative method theory including proofs of convergence and misconceptions and stochastic trace estimation that form the theoretical underpinnings of our work. Finally, we also included various examples of our package of analysis of both synthetic data and real data with modern neural network architectures.

\bibliography{example_paper}
\bibliographystyle{icml2019}

\newpage
\appendix
\section{Mathematical Definitions}
\begin{definition}{}
\label{def:wigner}
Let $\{Y_{i}\}$ and $\{Z_{ij}\}_{1\leq i\leq j}$ be two real-valued families of zero mean, i.i.d random variables, Furthermore suppose that $\mathbb{E}Z_{12}^{2}=1$ and for each $k \in \mathbb{N}$
\begin{equation}
    max(E|Z_{12}^{k},E|Y_{1}|^{k})< \infty
\end{equation}
Consider a $P \times P$ symmetric matrix $M_{P}$, whose entries are given by
\begin{equation}
  \begin{cases}
    M_{P}(i,i) = Y_{i}\\
    M_{P}(i,j) = Z_{ij} = M_{P}(j,i) ,& \text{if } x\geq 1 
\end{cases} 
\end{equation}
The Matrix $M_{P}$ is known as a real symmetric Wigner matrix.
\end{definition}
\section{Lanczos Algorithm Primer}
\label{sec:whateveryoneshouldknowlanczos}
\subsection{Why does anyone use Power Iterations?}
The Lanczos method be be explicitly derived by considering the optimization of the Rayleigh quotient \citep{golub2012matrix}
\begin{equation}
    r(\vv) = \frac{\vv^{T}\mH \vv}{\vv^{T}\vv}
\end{equation}
over the entire Krylov subspace $\mathcal{K}_{m}(\mH,\vv)$ as opposed to power iteration which is a particular vector in the Krylov subspace $\vu = \mH^{m}\vv$. Despite this, practitioners looking to learn the leading eigenvalue, very often resort to the power iteration, likely due to its implementational simplicity. We showcase the power iterations relative inferiority to Lanczos with the following convergence theorems
\begin{theorem}
	\label{theorem:lanczoseigenvalues}
	Let $H^{P\times P}$ be a symmetric matrix with eigenvalues $\lambda_{1}\geq .. \geq \lambda_{P}$ and corresponding orthonormal eigenvectors $z_{1},..z_{P}$. If $\theta_{1}\geq .. \geq \theta_{m}$ are the eigenvalues of the matrix $T_{m}$ obtained after $m$ Lanczos steps and $q_{1},...q_{m}$ the corresponding Ritz eigenvectors then
	\begin{equation}
	\label{eq:lanczosconv}
	\begin{aligned}
	& \lambda_{1} \geq \theta_{1} \geq \lambda_{1} - \frac{(\lambda_{1}-\lambda_{n})\tan^{2}(\theta_{1})}{(c_{m-1}(1+2\rho_{1}))^{2}}\\
	& \lambda_{P} \leq \theta_{k} \leq \lambda_{m} + \frac{(\lambda_{1}-\lambda_{n})\tan^{2}(\theta_{1})}{(c_{m-1}(1+2\rho_{1}))^{2}} \\
	\end{aligned}	
	\end{equation}
	
	where $c_{m}$ is the Chebyshev polyomial of order $k$. $\cos{\theta_{1}} = |\vq_{1}^{T}\vz_{1}|$ \& $\rho_{1}= (\lambda_{1}-\lambda_{2})/(\lambda_{2}-\lambda_{n})$
\end{theorem}
\begin{proof}
see \citep{golub2012matrix}.
\end{proof}
\begin{theorem}
	\label{theorem:powereigenvalues}
    Assuming the same notation as in Theorem \ref{theorem:lanczoseigenvalues}, after $m$ power iteration steps the corresponding extremal eigenvalue estimate is lower bounded by
	\begin{equation}
	\label{eq:powerconv}
	\begin{aligned}
	& \lambda_{1} \geq \theta_{1} \geq \lambda_{1} - (\lambda_{1}-\lambda_{n})\tan^{2}(\theta_{1})\bigg(\frac{\lambda_{2}}{\lambda_{1}}\bigg)^{2m-1}
	\end{aligned}
	\end{equation}
	
\end{theorem}
From the rapid growth of orthogonal polynomials such as Chebyshev, we expect Lanczos superiority to significantly emerge for larger spectral gap and iteration number. To verify this experimentally, we collect the non identical terms in the equations \ref{eq:lanczosconv} and \ref{eq:powerconv} of the lower bounds for $\lambda_{1}$ derived by Lanczos and Power iteration and denote them $L_{k-1}$ and $R_{k-1}$ respectively. For different values of $\lambda_{1}/\lambda_{2}$ and iteration number $m$ we give the ratio of these two quatities in Table \ref{table:lanczosvspower}. As can be clearly seen, the Lanczos lower bound is always closer to the true value, this improves with the iteration number $m$ and its relative edge is reduced if the spectral gap is decreased. 
\begin{table}[]
\begin{tabular}{@{}lllll@{}}
\toprule
\textbf{$\lambda_{1}/\lambda_{2}$} & \textbf{$m=5$} & \textbf{$m=10$} & \textbf{$m=15$} & \textbf{$m=20$}  \\ \midrule
 $1.5$ & $\frac{1.1\times 10^{-4}}{3.9 \times 10^{-2}}$ & $\frac{2\times 10^{-10}}{6.8 \times 10^{-4}}$ & $\frac{3.9\times 10^{-16}}{1.2 \times 10^{-5}}$ & $\frac{7.4\times 10^{-22}}{2.0 \times 10^{-7}}$   \\
 $1.1$ &  $\frac{2.7\times 10^{-2}}{4.7 \times 10^{-1}}$ &  $\frac{5.5\times 10^{-5}}{1.8 \times 10^{-1}}$ & $\frac{1.1\times 10^{-7}}{6.9 \times 10^{-2}}$ & $\frac{2.1\times 10^{-10}}{2.7 \times 10^{-2}}$  \\
 $1.01$ & $\frac{5.6\times 10^{-1}}{9.2 \times 10^{-1}}$ & $\frac{1.0\times 10^{-1}}{8.4 \times 10^{-1}}$  & $\frac{1.5\times 10^{-2}}{7.6 \times 10^{-1}}$ & $\frac{2.0\times 10^{-3}}{6.9 \times 10^{-1}}$   \\ \bottomrule
\end{tabular}
\caption{$L_{k-1}/R_{k-1}$ For different values of spectral gap $\lambda_{1}/\lambda_{2}$ and iteration number $m$, Table from \citep{golub2012matrix}}
\label{table:lanczosvspower}
\end{table}
\subsection{The Problem of Moments: Spectral Density Estimation Using Lanczos}
In this section we show that that the Lanczos Tri-Diagonal matrix corresponds to an orthogonal polynomial basis which matches the moments of $\vv^{T}\mH^{m}\vv$ and that when $\vv$ is a zero mean random vector with unit variance, this corresponds to the moment of the underlying spectral density.
\paragraph{Stochastic trace estimation}
\label{subsec:stochastictrace}
Using the expectation of quadratic forms, for zero mean, unit variance random vectors
\begin{equation}
\begin{aligned}
\label{eq:stochtrace}
    \mathbb{E}_{\vv}\text{Tr}(\vv^{T}\mH^{m}\vv) & =  \text{Tr}\mathbb{E}_{\vv}(\vv\vv^{T}\mH^{m}) = \text{Tr}(\mH^{m}) \\
    & = \sum_{i=1}^{P}\lambda_{i}^{m} = P \int_{\lambda \in \mathcal{D}} \lambda^{m} d\mu(\lambda) \\
\end{aligned}
\end{equation}
where we have used the linearity of trace and expectation. Hence in expectation over the set of random vectors, the trace of the inner product of $\vv$ and $\mH^{m}\vv$ is equal to the $m$'th moment of the spectral density of $\mH$.
\paragraph{Lanczos-Stieltjes}
\label{subsec:lanczos-stieltjes}
%A way to compute the coefficients of the three-term recurrence given the measure α is to approximate it by a discrete measure and to compute the coefficients of the recurrence corresponding to the discrete measure. Assume that we do not know the measure α but that we know the moments related to it. There are expressions directly relating the moments to the polynomial coef- ficients.
The Lanczos tri-diagonal matrix $\mT$ can be derived from the Moment matrix $\mM$, corresponding to the discrete measure $d\alpha(\lambda)$ satisfying the moments $\mu_{i} = \vv^{T}\mH^{i}\vv = \int \lambda^{i}d\alpha(\lambda)$ \citep{golub1994matrices}
\[
\mM = \begin{bmatrix}
1 & \vv^{T}\mH\vv  &\dots  & \vv^{T}\mH^{m-1}\vv\\
\vv^{T}\mH\vv & \vv^{T}\mH^{2}\vv & \ddots  & \vdots\\
\vdots & \ddots & \ddots & \vdots\\
\vv^{T}\mH^{m-1}\vv & \dots & \dots & \vv^{T}\mH^{2m-2}\vv
\end{bmatrix}
\]
% \begin{equation}
% (\mM)_{ij} = \vv^{T}\mH^{i+j-2}\vv
% \end{equation}
and hence for a zero mean unit variance initial seed vector, the eigenvector/eigenvalue pairs of $\mT$ contain information about the spectral density of $\mH$ as shown in section \ref{subsec:stochastictrace}. This is given by the following Theorem
\begin{theorem}
	\label{theorem:lanczosspectrum}
	The eigenvalues of $T_{k}$ are the nodes $t_{j}$ of the Gauss quadrature rule, the weights $w_{j}$ are the squares of the first elements of the normalized eigenvectors of $T_{k}$
\end{theorem}
\begin{proof}
See \citet{golub1994matrices}
\end{proof}
A quadrature rule is a relation of the form,
	\begin{equation}
		\label{eq:quadraturerule}
		\int_{a}^{b}f(\lambda)d\mu(\lambda) = \sum_{j=1}^{M}\rho_{j}f(t_{j})+R[f]
	\end{equation}
	for a function $f$, such that its Riemann-Stieltjes integral and all the moments exist on the measure $d\mu(\lambda)$, on the interval $[a,b]$ and where $R[f]$ denotes the unknown remainder. The first term on the RHS of \eqref{eq:quadraturerule} using Theorem \ref{theorem:lanczosspectrum} can be seen as a discrete approximation to the spectral density matching the first $m$ moments $v^{T}H^{m}v$ \citep{golub1994matrices,golub2012matrix}

For $n_{v}$ starting vectors, the corresponding discrete spectral density is given as
\begin{equation}
\label{eq:spectraldensitymanyseeds}
p(\lambda) = \frac{1}{n_{v}}\sum_{l=1}^{n_{v}}\biggl(\sum_{k=1}^{m}(\tau_{k}^{(l)})^{2}\delta(\lambda-\lambda_{k}^{(l)})\biggr),
\end{equation}
where $\tau_{k}^{(l)}$ corresponds to the first entry of the eigenvector of the $k$-th eigenvalue, $\lambda_{k}$, of the Lanczos tri-diagonal matrix, $\mT$, for the $l$-th starting vector \citep{ubaruapplications, lin2016approximating}. 
% In the event that we use a single seed vector, we obtain the simplified expression,
% \begin{equation}
% \label{eq:spectraldensity1seed}
% p(\lambda) = \sum_{k=1}^{m}(\tau_{k})^{2}\delta(\lambda-\lambda_{k}).
% \end{equation}
% \paragraph{Number of trace vectors used}
% The error between the expectation over the set of all zero mean, unit variance vectors $v$ and the monte carlo sum used in practice can be bounded \citep{hutchinson1990stochastic,roosta2015improved}. However these bounds are very loose, with typical matrices requiring hundreds of thousands of random vectors to guarantee a reasonable guarantee on fractional error \citep{granziol2018vbald} with authors typically using a vastly reduced number of random vectors in practice \citep{granziol2019meme,granziol2018vbald,granziol2018entropic,Granziol2017,papyan2018full,ghorbani2019investigation}.
% However in the high dimensional regime $N \rightarrow \infty$, we expect the squared overlap of each random vector with an eigenvector of $H$, $|v^{T}\phi_{i}|^{2} \approx \frac{1}{N} \forall i$, with high probability. This result can be seen by computing the moments of the overlap between Rademacher vectors, containing elements $P(v_{j} = \pm 1) = 0.5$. Further analytical results for Gaussian vectors have been obtained \citep{cai2013distributions}. 

\subsection{Computational Complexity}
For large matrices, the computational complexity of the algorithm depends on the Hessian vector product, which for neural networks is $\mathcal{O}(mNP)$ where $P$ denotes the number of parameters in the network, $m$ is the number of Lanczos iterations and $N$ is the number of data-points. The full re-orthogonalisation adds two matrix vector products, which is of cost $\mathcal{O}(m^{2}P)$, where typically $m^{2}\ll N$. Each random vector used can be seen as another full run of the Lanczos algorithm, so for $d$ random vectors the total complexity is $\mathcal{O}(dmP(N+m))$

\paragraph{Importance of keeping orthogonality}
The update equations of the Lanczos algorithm lead to a tri-diagonal matrix $\mT = \R^{m\times m}$, whose eigenvalues represent the approximated eigenvalues of the matrix $\mH$ and whose eigenvectors, when projected back into the the Krylov-subspace, $\mathscr{K}(\mH,\vv)$, give the approximated eigenvectors of $\mH$. In finite precision, it is known \citep{meurant2006lanczos} that the Lanczos algorithm fails to maintain orthogonality between its Ritz vectors, with corresponding convergence failure. In order to remedy this, we re-orthonormalise at each step \citep{bai1996some} (as shown in line $9$ of Algorithm \ref{alg:lanczos}) and observe a high degree of orthonormality between the Ritz eigenvectors. Orthonormality is also essential for achieving accurate spectral resolution as the Ritz value weights are given by the squares of the first elements of the normalised eigenvectors. For the practitioner wishing to reduce the computational cost of maintaining orthogonality, there exist more elaborate schemes \citep{meurant2006lanczos,golub1994matrices}.
\begin{algorithm}[tb]
	\caption{Lanczos Algorithm}
	\label{alg:lanczos}
	
	\begin{algorithmic}[1]
		\STATE {\bfseries Input:} Hessian vector product $\{\mH\vv\}$, number of steps $m$
		\STATE {\bfseries Output:} Ritz eigenvalue/eigenvector pairs $\{\lambda_{i},\vu_{i}\}$ \& quadrature weights $\tau_{i}$
		\STATE Set $\vv := \vv/\sqrt{(\vv^{T}\vv)}$
		\STATE Set $\beta := 0$, $\vv_{old}: = \vv$
		\STATE Set $\mV(:,1) := \vv$
		\FOR {$j$ in $1,..,m$}
		\STATE $\vw = \mH\vv-\beta\vv_{old}$
		\STATE $\mT(j,j) = \alpha = \vw^{T}\vv$
		\STATE $w = \vw -\alpha\vw -\mV \mV^{T}\vw$
		\STATE $\beta = \sqrt{\vw^{T}\vw}$
		\STATE $\vv_{old} = \vv$
		\STATE $\vv = \vw/\beta$
		\STATE $\mV(:,j+1) = \vv$
		\STATE $\mT(j,j+1) = \mT(j+1,1) = \beta$
		\ENDFOR
		%		%\STATE $$
		\STATE $\{\lambda_{i},\ve{_{i}}\} = eig(\mT)$ 
		\STATE $\vu_{i} = \mV \ve_{i}$
		\STATE $\tau_{i} = (\ve_{i}^{T}[1,0,0...0])^{2}$
	\end{algorithmic}
\end{algorithm}

\section{Code Run}
\label{sec:dnncoderun}
In our interface, to train the network, we call:

\begin{lstlisting}[language=Python, basicstyle=\small, breaklines=true]
train_network(
    dir='result/VGG16-CIFAR100/',
    dataset='CIFAR100',
    # dataset='CIFAR10', if testing on CIFAR-10 data-set.
    data_path='data/'
    data_path='data/',
    epochs=300,
    model='VGG16' 
    # model='PreResNet110', if training with Preactivated ResNet with 110 layers instead.
    optimizer='SGD',
    schedule='linear',
    # This will direct the learning rate schedule to be the linear schedule defined in Equation 19.
    optimizer_kwargs={
        'lr': 0.05,
        'momentum': 0.9,
        'weight_decay': 5e-4})
\end{lstlisting}
\subsection{Running the Example C100}
\label{subsec:example}

To replicate this example, use the following command:

\begin{lstlisting}[language=Python, basicstyle=\small, breaklines=true]
compute_eigenspectrum(
    dataset='CIFAR100',
    # dataset='CIFAR10', if testing on CIFAR-10 data-set.
    data_path='data/'
    data_path='data/',
    model='VGG16',
    checkpoint_path='result/VGG16-CIFAR100/checkpoint-00300.pt',
    save_spectrum_path='result/VGG16-CIFAR100/spectra/spectrum-00300-ggn_lanczos',
    # change accordingly, if considering the Hessian matrix
    save_eigvec=True,
    lanczos_iters=100,
    curvature_matrix='ggn_lanczos',
    # curvature_matrix='hessian_lanczos', if considering the Hessian matrix instead
)
\end{lstlisting}
\subsection{Running the Example C10}
\label{subsec:example2}
To replicate this result, use the following command:

\begin{lstlisting}[language=Python, basicstyle=\small, breaklines=true]
build_loss_landscape(
    dataset='CIFAR100',
    # dataset='CIFAR10', if testing on CIFAR-10 data-set.
    data_path='data/',
    model='VGG16',
    # Change accordingly for other architectures such as PreResNet110
    dist=0.25,
    n_points=21,
    spectrum_path='result/VGG16-CIFAR100/spectra/spectrum-00300-ggn_lanczos',
    # change accordingly for the Hessian result
    checkpoint_path='result/VGG16-CIFAR100/checkpoint-00300.pt',
    save_path='result/VGG16-CIFAR100/losslandscape-00300.npz'
)

plot_loss_landscape('result/VGG16-CIFAR100/losslandscape-00300.npz')
plt.show()
\end{lstlisting}

\section{An Illustrated Example}
\label{appendix:example}
We give an illustration on an example of using the MLRG-DeepCurvature package and more details, including detailed documentation of each user function and the output of this particular example, can be found at our open-source repository. We begin by importing the necessary functions and packages:
\begin{lstlisting}[language=Python, basicstyle=\small]
from core import *
from visualise import *
import matplotlib.pyplot as plt
\end{lstlisting}

In this example, we train a VGG16 network on CIFAR-100 for 100 epochs. In a test computer with AMD Ryzen 3700X CPU and NVIDIA GeForce RTX 2080 Ti GPU, each epoch of training takes less than 10 seconds:

\begin{lstlisting}[language=Python, basicstyle=\small]
train_network(
    dir='result/VGG16-CIFAR100/',
    dataset='CIFAR100',
    data_path='data/',
    epochs=100,
    model='VGG16',
    optimizer='SGD',
    optimizer_kwargs={
        'lr': 0.03,
        'momentum': 0.9,
        'weight_decay': 5e-4
    }
)
\end{lstlisting}

This step generates a number of training statistics files (starting with \texttt{stats-}) and checkpoint files (\texttt{checkpoint-00XXX.pt}, where \texttt{XXX} is the epoch number) that contain the \texttt{state\_dict} of the optimizer and the model. We may additionally visualise the training processes by looking at the basic statistics by calling \texttt{plot\_training} function. With the checkpoints generated, we may now compute analyse the eigenspectrum of the curvature matrix evaluated at the desired point of training. For example, if we would like to evaluate the Generalised Gauss-Newton (GGN) matrix at the end of the training with 20 Lanczos iterations, we call:

\begin{lstlisting}[language=Python, basicstyle=\small, breaklines=true]
compute_eigenspectrum(
    dataset='CIFAR100',
    data_path='data/',
    model='VGG16',
    checkpoint_path='result/VGG16-CIFAR100/checkpoint-00100.pt',
    save_spectrum_path='result/VGG16-CIFAR100/spectra/spectrum-00100-ggn_lanczos',
    save_eigvec=True,
    lanczos_iters=20,
    curvature_matrix='ggn_lanczos',
)
\end{lstlisting}
This function call saves the spectrum results (including eigenvalues, eigenvectors and other related statistics) in the \texttt{save\_spectrum\_path} path string defined. To visualise the spectrum as stem plot similar to Figure \ref{fig:diagggnmcvslanc}, we simply call:

\begin{lstlisting}[language=Python, basicstyle=\small, breaklines=true]
plot_spectrum('lanczos', path='result/VGG16-CIFAR100/spectra/spectrum-00100-ggn_lanczos.npz')
plt.show()
\end{lstlisting}

Finally, with the eigenvalues and eigenvectors computed, we might be interested in knowing how sensitive the network is to perturbation along these directions. To achieve this, we first construct a loss landscape by setting the number of query points and maximum perturbation to apply. To achieve that, we call:

\begin{lstlisting}[language=Python, basicstyle=\small, breaklines=true]
build_loss_landscape(
    dataset='CIFAR100',
    data_path='data/',
    model='VGG16',
    dist=1.,
    n_points=21,
    spectrum_path='result/VGG16-CIFAR100/spectra/spectrum-00100-ggn_lanczos',
    checkpoint_path='result/VGG16-CIFAR100/checkpoint-00100.pt',
    save_path='result/VGG16-CIFAR100/losslandscape-00100.npz'
)

plot_loss_landscape('result/VGG16-CIFAR100/losslandscape-00100.npz')
plt.show()
\end{lstlisting}

where in this example, we set the maximum perturbation to be 1 (\texttt{dist} argument) and number of query points along each direction to be 21 (\texttt{n\_points} argument). This will produce diagrams similar to Figure \ref{fig:vgg16c100losssurface} that show the effect of perturbation in the loss and accuracy for both training and testing.

\end{document}

%% file: math_commands.tex
\usepackage{amsmath,amsfonts,bm}

\def\eqref#1{equation~\ref{#1}}
\def\1{\bm{1}}

\def\vd{{\bm{d}}}
\def\ve{{\bm{e}}}

\def\vp{{\bm{p}}}
\def\vq{{\bm{q}}}

\def\vu{{\bm{u}}}
\def\vv{{\bm{v}}}
\def\vw{{\bm{w}}}

\def\vz{{\bm{z}}}

\def\mD{{\bm{D}}}

\def\mH{{\bm{H}}}
\def\mI{{\bm{I}}}

\def\mM{{\bm{M}}}

\def\mT{{\bm{T}}}
\def\mU{{\bm{U}}}
\def\mV{{\bm{V}}}

\def\mX{{\bm{X}}}
\def\mY{{\bm{Y}}}

\DeclareMathAlphabet{\mathsfit}{\encodingdefault}{\sfdefault}{m}{sl}
\SetMathAlphabet{\mathsfit}{bold}{\encodingdefault}{\sfdefault}{bx}{n}

\newcommand{\R}{\mathbb{R}}

\newcommand{\Var}{\mathrm{Var}}

\DeclareMathOperator{\Tr}{Tr}

%% file: UAI.bbl
\begin{thebibliography}{44}
\providecommand{\natexlab}[1]{#1}
\providecommand{\url}[1]{\texttt{#1}}
\expandafter\ifx\csname urlstyle\endcsname\relax
  \providecommand{\doi}[1]{doi: #1}\else
  \providecommand{\doi}{doi: \begingroup \urlstyle{rm}\Url}\fi

\bibitem[Abadi et~al.(2016)Abadi, Barham, Chen, Chen, Davis, Dean, Devin,
  Ghemawat, Irving, Isard, et~al.]{abadi2016tensorflow}
Abadi, M., Barham, P., Chen, J., Chen, Z., Davis, A., Dean, J., Devin, M.,
  Ghemawat, S., Irving, G., Isard, M., et~al.
\newblock Tensorflow: A system for large-scale machine learning.
\newblock In \emph{12th $\{$USENIX$\}$ Symposium on Operating Systems Design
  and Implementation ($\{$OSDI$\}$ 16)}, pp.\  265--283, 2016.

\bibitem[Akemann et~al.(2011)Akemann, Baik, and
  Di~Francesco]{akemann2011oxford}
Akemann, G., Baik, J., and Di~Francesco, P.
\newblock \emph{The Oxford handbook of random matrix theory}.
\newblock Oxford University Press, 2011.

\bibitem[Anonymous(2020)]{anonymous2020towards}
Anonymous.
\newblock Towards understanding the true loss surface of deep neural networks
  using random matrix theory and iterative spectral methods.
\newblock In \emph{Submitted to International Conference on Learning
  Representations}, 2020.
\newblock URL \url{https://openreview.net/forum?id=H1gza2NtwH}.
\newblock under review.

\bibitem[Bai et~al.(1996)Bai, Fahey, and Golub]{bai1996some}
Bai, Z., Fahey, G., and Golub, G.
\newblock Some large-scale matrix computation problems.
\newblock \emph{Journal of Computational and Applied Mathematics}, 74\penalty0
  (1-2):\penalty0 71--89, 1996.

\bibitem[Bishop(2006)]{bishop2006pattern}
Bishop, C.~M.
\newblock \emph{Pattern recognition and machine learning}.
\newblock springer, 2006.

\bibitem[Chaudhari et~al.(2016)Chaudhari, Choromanska, Soatto, LeCun, Baldassi,
  Borgs, Chayes, Sagun, and Zecchina]{chaudhari2016entropy}
Chaudhari, P., Choromanska, A., Soatto, S., LeCun, Y., Baldassi, C., Borgs, C.,
  Chayes, J., Sagun, L., and Zecchina, R.
\newblock Entropy-{SGD}: Biasing gradient descent into wide valleys.
\newblock \emph{arXiv preprint arXiv:1611.01838}, 2016.

\bibitem[Chollet(2015)]{chollet2015}
Chollet, F.
\newblock Keras.
\newblock \url{https://github.com/fchollet/keras}, 2015.

\bibitem[Choromanska et~al.(2015{\natexlab{a}})Choromanska, Henaff, Mathieu,
  Arous, and LeCun]{choromanska2015loss}
Choromanska, A., Henaff, M., Mathieu, M., Arous, G.~B., and LeCun, Y.
\newblock The loss surfaces of multilayer networks.
\newblock In \emph{Artificial Intelligence and Statistics}, pp.\  192--204,
  2015{\natexlab{a}}.

\bibitem[Choromanska et~al.(2015{\natexlab{b}})Choromanska, LeCun, and
  Arous]{choromanska2015open}
Choromanska, A., LeCun, Y., and Arous, G.~B.
\newblock Open problem: The landscape of the loss surfaces of multilayer
  networks.
\newblock In \emph{Conference on Learning Theory}, pp.\  1756--1760,
  2015{\natexlab{b}}.

\bibitem[Dangel et~al.(2019)Dangel, Kunstner, and Hennig]{dangel2019backpack}
Dangel, F., Kunstner, F., and Hennig, P.
\newblock Backpack: Packing more into backprop.
\newblock \emph{arXiv preprint arXiv:1912.10985}, 2019.

\bibitem[Dauphin et~al.(2014)Dauphin, Pascanu, Gulcehre, Cho, Ganguli, and
  Bengio]{dauphin2014identifying}
Dauphin, Y.~N., Pascanu, R., Gulcehre, C., Cho, K., Ganguli, S., and Bengio, Y.
\newblock Identifying and attacking the saddle point problem in
  high-dimensional non-convex optimization.
\newblock In \emph{Advances in neural information processing systems}, pp.\
  2933--2941, 2014.

\bibitem[Gardner et~al.(2018)Gardner, Pleiss, Weinberger, Bindel, and
  Wilson]{gardner2018gpytorch}
Gardner, J., Pleiss, G., Weinberger, K.~Q., Bindel, D., and Wilson, A.~G.
\newblock Gpytorch: Blackbox matrix-matrix {G}aussian process inference with
  {GPU} acceleration.
\newblock In \emph{Advances in Neural Information Processing Systems}, pp.\
  7576--7586, 2018.

\bibitem[Ghorbani et~al.(2019)Ghorbani, Krishnan, and
  Xiao]{ghorbani2019investigation}
Ghorbani, B., Krishnan, S., and Xiao, Y.
\newblock An investigation into neural net optimization via {H}essian
  eigenvalue density.
\newblock \emph{arXiv preprint arXiv:1901.10159}, 2019.

\bibitem[Golub \& Meurant(1994)Golub and Meurant]{golub1994matrices}
Golub, G.~H. and Meurant, G.
\newblock Matrices, moments and quadrature.
\newblock \emph{Pitman Research Notes in Mathematics Series}, pp.\  105--105,
  1994.

\bibitem[Golub \& Van~Loan(2012)Golub and Van~Loan]{golub2012matrix}
Golub, G.~H. and Van~Loan, C.~F.
\newblock \emph{Matrix computations}, volume~3.
\newblock JHU press, 2012.

\bibitem[Granziol et~al.(2019)Granziol, Ru, Zohren, Dong, Osborne, and
  Roberts]{granziol2019meme}
Granziol, D., Ru, B., Zohren, S., Dong, X., Osborne, M., and Roberts, S.
\newblock Meme: An accurate maximum entropy method for efficient approximations
  in large-scale machine learning.
\newblock \emph{Entropy}, 21\penalty0 (6):\penalty0 551, 2019.

\bibitem[He et~al.(2019)He, Huang, and Yuan]{he2019asymmetric}
He, H., Huang, G., and Yuan, Y.
\newblock Asymmetric valleys: Beyond sharp and flat local minima.
\newblock \emph{arXiv preprint arXiv:1902.00744}, 2019.

\bibitem[He et~al.(2016)He, Zhang, Ren, and Sun]{he2016deep}
He, K., Zhang, X., Ren, S., and Sun, J.
\newblock Deep residual learning for image recognition.
\newblock In \emph{Proceedings of the IEEE conference on computer vision and
  pattern recognition}, pp.\  770--778, 2016.

\bibitem[Hutchinson(1990)]{hutchinson1990stochastic}
Hutchinson, M.~F.
\newblock A stochastic estimator of the trace of the influence matrix for
  {L}aplacian smoothing splines.
\newblock \emph{Communications in Statistics-Simulation and Computation},
  19\penalty0 (2):\penalty0 433--450, 1990.

\bibitem[Ioffe \& Szegedy(2015)Ioffe and Szegedy]{ioffe2015batch}
Ioffe, S. and Szegedy, C.
\newblock Batch normalization: Accelerating deep network training by reducing
  internal covariate shift.
\newblock \emph{arXiv preprint arXiv:1502.03167}, 2015.

\bibitem[Izmailov et~al.(2018)Izmailov, Podoprikhin, Garipov, Vetrov, and
  Wilson]{izmailov2018averaging}
Izmailov, P., Podoprikhin, D., Garipov, T., Vetrov, D., and Wilson, A.~G.
\newblock Averaging weights leads to wider optima and better generalization.
\newblock \emph{Uncertainty in Artificial Intelligence (UAI)}, 2018.

\bibitem[Izmailov et~al.(2019)Izmailov, Maddox, Kirichenko, Garipov, Vetrov,
  and Wilson]{izmailov2019subspace}
Izmailov, P., Maddox, W.~J., Kirichenko, P., Garipov, T., Vetrov, D., and
  Wilson, A.~G.
\newblock Subspace inference for bayesian deep learning.
\newblock \emph{Uncertainty in Artificial Intelligence (UAI)}, 2019.

\bibitem[Kohler et~al.(2018)Kohler, Daneshmand, Lucchi, Zhou, Neymeyr, and
  Hofmann]{kohler2018exponential}
Kohler, J., Daneshmand, H., Lucchi, A., Zhou, M., Neymeyr, K., and Hofmann, T.
\newblock Exponential convergence rates for batch normalization: The power of
  length-direction decoupling in non-convex optimization.
\newblock \emph{arXiv preprint arXiv:1805.10694}, 2018.

\bibitem[Li et~al.(2017)Li, Xu, Taylor, and Goldstein]{li2017visualizing}
Li, H., Xu, Z., Taylor, G., and Goldstein, T.
\newblock Visualizing the loss landscape of neural nets.
\newblock \emph{arXiv preprint arXiv:1712.09913}, 2017.

\bibitem[Lin et~al.(2016)Lin, Saad, and Yang]{lin2016approximating}
Lin, L., Saad, Y., and Yang, C.
\newblock Approximating spectral densities of large matrices.
\newblock \emph{SIAM Review}, 58\penalty0 (1):\penalty0 34--65, 2016.

\bibitem[Maddox et~al.(2019)Maddox, Izmailov, Garipov, Vetrov, and
  Wilson]{maddox2019simple}
Maddox, W.~J., Izmailov, P., Garipov, T., Vetrov, D.~P., and Wilson, A.~G.
\newblock A simple baseline for bayesian uncertainty in deep learning.
\newblock In \emph{Advances in Neural Information Processing Systems}, pp.\
  13132--13143, 2019.

\bibitem[Marchenko \& Pastur(1967)Marchenko and
  Pastur]{marchenko1967distribution}
Marchenko, V.~A. and Pastur, L.~A.
\newblock Distribution of eigenvalues for some sets of random matrices.
\newblock \emph{Matematicheskii Sbornik}, 114\penalty0 (4):\penalty0 507--536,
  1967.

\bibitem[Martens(2016)]{Martens2016}
Martens, J.
\newblock \emph{Second-order optimization for neural networks}.
\newblock PhD thesis, University of Toronto, 2016.
\newblock URL
  \url{http://www.cs.toronto.edu/~jmartens/docs/thesis_phd_martens.pdf}.

\bibitem[Martens \& Sutskever(2012)Martens and Sutskever]{martens2012training}
Martens, J. and Sutskever, I.
\newblock Training deep and recurrent networks with {H}essian-free
  optimization.
\newblock In \emph{Neural networks: Tricks of the trade}, pp.\  479--535.
  Springer, 2012.

\bibitem[Meurant \& Strako{\v{s}}(2006)Meurant and
  Strako{\v{s}}]{meurant2006lanczos}
Meurant, G. and Strako{\v{s}}, Z.
\newblock The {L}anczos and conjugate gradient algorithms in finite precision
  arithmetic.
\newblock \emph{Acta Numerica}, 15:\penalty0 471--542, 2006.

\bibitem[Nesterov(2013)]{nesterov2013introductory}
Nesterov, Y.
\newblock \emph{Introductory lectures on convex optimization: A basic course},
  volume~87.
\newblock Springer Science \& Business Media, 2013.

\bibitem[Paszke et~al.(2017)Paszke, Gross, Chintala, Chanan, Yang, DeVito, Lin,
  Desmaison, Antiga, and Lerer]{paszke2017automatic}
Paszke, A., Gross, S., Chintala, S., Chanan, G., Yang, E., DeVito, Z., Lin, Z.,
  Desmaison, A., Antiga, L., and Lerer, A.
\newblock Automatic differentiation in {P}ytorch.
\newblock 2017.

\bibitem[Pearlmutter(1994)]{pearlmutter1994fast}
Pearlmutter, B.~A.
\newblock Fast exact multiplication by the {H}essian.
\newblock \emph{Neural computation}, 6\penalty0 (1):\penalty0 147--160, 1994.

\bibitem[Pennington \& Bahri(2017)Pennington and Bahri]{pennington2017geometry}
Pennington, J. and Bahri, Y.
\newblock Geometry of neural network loss surfaces via random matrix theory.
\newblock In \emph{Proceedings of the 34th International Conference on Machine
  Learning-Volume 70}, pp.\  2798--2806. JMLR. org, 2017.

\bibitem[Ritter et~al.(2018)Ritter, Botev, and Barber]{ritter2018a}
Ritter, H., Botev, A., and Barber, D.
\newblock A scalable laplace approximation for neural networks.
\newblock In \emph{International Conference on Learning Representations}, 2018.
\newblock URL \url{https://openreview.net/forum?id=Skdvd2xAZ}.

\bibitem[Sagun et~al.(2016)Sagun, Bottou, and LeCun]{sagun2016eigenvalues}
Sagun, L., Bottou, L., and LeCun, Y.
\newblock Eigenvalues of the {H}essian in deep learning: Singularity and
  beyond.
\newblock \emph{arXiv preprint arXiv:1611.07476}, 2016.

\bibitem[Sagun et~al.(2017)Sagun, Evci, Guney, Dauphin, and
  Bottou]{sagun2017empirical}
Sagun, L., Evci, U., Guney, V.~U., Dauphin, Y., and Bottou, L.
\newblock Empirical analysis of the {H}essian of over-parametrized neural
  networks.
\newblock \emph{arXiv preprint arXiv:1706.04454}, 2017.

\bibitem[Santurkar et~al.(2018)Santurkar, Tsipras, Ilyas, and
  Madry]{santurkar2018does}
Santurkar, S., Tsipras, D., Ilyas, A., and Madry, A.
\newblock How does batch normalization help optimization?
\newblock In \emph{Advances in Neural Information Processing Systems}, pp.\
  2483--2493, 2018.

\bibitem[Simonyan \& Zisserman(2014)Simonyan and Zisserman]{simonyan2014very}
Simonyan, K. and Zisserman, A.
\newblock Very deep convolutional networks for large-scale image recognition.
\newblock \emph{arXiv preprint arXiv:1409.1556}, 2014.

\bibitem[Tao(2012)]{tao2012topics}
Tao, T.
\newblock \emph{Topics in random matrix theory}, volume 132.
\newblock American Mathematical Soc., 2012.

\bibitem[Ubaru \& Saad()Ubaru and Saad]{ubaruapplications}
Ubaru, S. and Saad, Y.
\newblock Applications of trace estimation techniques.

\bibitem[Vinyals \& Povey(2012)Vinyals and Povey]{vinyals2012krylov}
Vinyals, O. and Povey, D.
\newblock Krylov subspace descent for deep learning.
\newblock In \emph{Artificial Intelligence and Statistics}, pp.\  1261--1268,
  2012.

\bibitem[Wigner(1993)]{wigner1993characteristic}
Wigner, E.~P.
\newblock Characteristic vectors of bordered matrices with infinite dimensions
  i.
\newblock In \emph{The Collected Works of Eugene Paul Wigner}, pp.\  524--540.
  Springer, 1993.

\bibitem[Yao et~al.(2018)Yao, Gholami, Lei, Keutzer, and
  Mahoney]{yao2018hessian}
Yao, Z., Gholami, A., Lei, Q., Keutzer, K., and Mahoney, M.~W.
\newblock {H}essian-based analysis of large batch training and robustness to
  adversaries.
\newblock In \emph{Advances in Neural Information Processing Systems}, pp.\
  4949--4959, 2018.

\end{thebibliography}
